\documentclass{article}

\usepackage[nonatbib,final]{neurips_2020}
\usepackage[numbers]{natbib}

\usepackage[utf8]{inputenc} 
\usepackage[T1]{fontenc}    
\usepackage{hyperref}       
\usepackage{url}            
\usepackage{booktabs}       
\usepackage{amsfonts}       
\usepackage{nicefrac}       
\usepackage{microtype}      
\usepackage{amsmath}
\usepackage{amssymb}
\usepackage{amsthm}
\usepackage{mathrsfs}
\usepackage{comment}
\usepackage{bm}
\usepackage{multirow}
\usepackage{placeins}
\usepackage{cleveref}
\usepackage{microtype}
\usepackage{graphicx}
\usepackage{subcaption}

\newtheorem{lemma}{Lemma}
\newtheorem{definition}{Definition}

\title{Multimodal Generative Learning Utilizing Jensen-Shannon-Divergence}

\author{%
  Thomas M.~Sutter, \ \ Imant Daunhawer, \ \ Julia E. Vogt\\
  Department of Computer Science\\
  ETH Zurich\\
  \texttt{\{thomas.sutter,imant.daunhawer,julia.vogt\}@inf.ethz.ch} \\  
}

\begin{document}

\maketitle

\begin{abstract}
Learning from different data types is a long-standing goal in machine learning research, as multiple information sources co-occur when describing natural phenomena. However, existing generative models that approximate a multimodal ELBO rely on difficult or inefficient training schemes to learn a joint distribution and the dependencies between modalities.
In this work, we propose a novel, efficient objective function that utilizes the Jensen-Shannon divergence for multiple distributions. It simultaneously  approximates the unimodal and joint multimodal posteriors directly via a dynamic prior. In addition, we theoretically prove that the new multimodal JS-divergence (mmJSD) objective optimizes an ELBO.
In extensive experiments, we demonstrate the advantage of the proposed mmJSD model compared to previous work in unsupervised, generative learning tasks.
\end{abstract}

\section{Introduction}
\label{sec:intro}
Replicating the human ability to process and relate information coming from different sources and learn from these is a long-standing goal in machine learning \cite{baltruvsaitis2018multimodal}. Multiple information sources offer the potential of learning better and more generalizable representations, but pose challenges at the same time: models have to be aware of complex intra- and inter-modal relationships, and be robust to missing modalities \cite{ngiam2011multimodal,zadeh2017tensor}. However, the excessive labelling of multiple data types is expensive and hinders possible applications of fully-supervised approaches~\cite{fang2015captions,karpathy2015deep}. Simultaneous observations of multiple modalities moreover provide self-supervision in the form of shared information which connects the different modalities. Self-supervised, generative models are a promising approach to capture this joint distribution and flexibly support missing modalities with no additional labelling cost attached.
Based on the shortcomings of previous work (see \Cref{subsec:related_work}), we formulate the following wish-list for multimodal, generative models:

\textbf{Scalability. }
The model should be able to efficiently handle any number of modalities. Translation approaches \cite{huang2018multimodal,Zhu_2017_ICCV} have had great success in combining two modalities and translating from one to the other. However, the training of these models is computationally expensive for more than two modalities due to the exponentially growing number of possible paths between subsets of modalities. 

\textbf{Missing data.}
A multimodal method should be robust to missing data and handle any combination of available and missing data types. For discriminative tasks, the loss in performance should be minimized. For generation, the estimation of missing data types should be conditioned on and coherent with available data while providing diversity over modality-specific attributes in the generated samples.

\textbf{Information gain.}
Multimodal models should benefit from multiple modalities for discriminative as well as for generative tasks.


In this work, we introduce a novel probabilistic, generative and self-supervised multi-modal model. The proposed model is able to integrate information from different modalities, reduce uncertainty and ambiguity in redundant sources, as well as handle missing modalities while making no assumptions about the nature of the data, especially about the inter-modality relations.

We base our approach directly in the Variational Bayesian Inference framework and propose the new multimodal Jensen-Shannon divergence (mmJSD) objective. We introduce the idea of a dynamic prior for multimodal data, which enables the use of the Jensen-Shannon divergence for $M$ distributions \cite{aslam2007query,lin1991divergence} and interlinks the unimodal probabilistic representations of the $M$ observation types. Additionally, we are - to the best of our knowledge - the first to empirically show the advantage of modality-specific subspaces for multiple data types in a self-supervised and scalable setting. For the experiments, we concentrate on Variational Autoencoders \cite{Kingma2013}. In this setting, our multimodal extension to variational inference implements a scalable method, capable of handling missing observations, generating coherent samples and learning meaningful representations. We empirically show this on two different datasets. In the context of scalable generative models, we are the first to perform experiments on datasets with more than 2 modalities showing the ability of the proposed method to perform well in a setting with multiple modalities.

\section{Theoretical Background \& Related Work}
\label{sec:background}
We consider some dataset of $N$ i.i.d. sets $\{\bm{X}^{(i)} \}_{i=1}^N$ with every $\bm{X}^{(i)}$ being a set of $M$ modalities $\bm{X}^{(i)} = \{\bm{x}_j^{(i)}\}_{j=1}^M$. We assume that the data is generated by some random process involving a joint hidden random variable $\bm{z}$ where inter-modality dependencies are unknown. In general, the same assumptions are valid as in the unimodal setting \cite{Kingma2013}. The marginal log-likelihood can be decomposed into a sum over marginal log-likelihoods of individual sets $\log p_{\theta}(\{\bm{X}^{(i)} \}_{i=1}^N) = \sum_{i=1}^{N} \log p_{\theta} (\bm{X}^{(i)})$, which can be written as:
\begin{align}
\label{eq:mm_likelhood}
    \log p_{\theta}(\bm{X}^{(i)}) = & KL ( q_\phi(\bm{z} | \bm{X}^{(i)}) || p_\theta(\bm{z} | \bm{X}^{(i)})) + \mathcal{L}(\theta, \phi ; \bm{X}^{(i)}), 
\end{align}
\begin{align}
\label{eq:elbo_mm}
\textnormal{with}\,\,	\mathcal{L}(\theta, \phi ; \bm{X}^{(i)}) :=& E_{q_{\phi}(\bm{z}|\bm{X})}[\log p_{\theta}(\bm{X}^{(i)}|\bm{z})] - KL(q_\phi(\bm{z} | \bm{X}^{(i)}) || p_\theta(\bm{z})).
\end{align}
$\mathcal{L}(\theta, \phi ; \bm{X}^{(i)})$ is called evidence lower bound (ELBO) on the marginal log-likelihood of set $i$. The ELBO forms a computationally tractable objective to approximate the joint data distribution $\log p_{\theta}(\bm{X}^{(i)})$ which can be efficiently optimized, because it follows from the non-negativity of the KL-divergence: $\log p_{\theta}(\bm{X}^{(i)}) \geq  \mathcal{L}(\theta, \phi ; \bm{X}^{(i)})$.
Particular to the multimodal case is what happens to the ELBO formulation if one or more data types are missing: we are only able to approximate the true posterior $p_{\theta}(\bm{z}|\bm{X}^{(i)})$ by the variational function $q_{\phi_K}(\bm{z}|\bm{X}_K^{(i)})$. $\bm{X}_K^{(i)}$ denotes a subset of $\bm{X}^{(i)}$ with  $K$ available modalities where $K \leq M$. However, we would still like to be able to approximate the true multimodal posterior distribution $p_\theta(\bm{z}| \bm{X}^{(i)})$ of all data types. For simplicity, we always use $\bm{X}_K^{(i)}$ to symbolize missing data for set $i$, although there is no information about which or how many modalities are missing. Additionally, different modalities might be missing for different sets $i$. In this case, the ELBO formulation changes accordingly:
\begin{align}
\label{eq:elbo_mm_missing}
\mathcal{L}_K(\theta, \phi_K ; \bm{X}^{(i)}) := & E_{q_{\phi_K}(\bm{z}|\bm{X}_K^{(i)})}[\log(p_\theta(\bm{X}^{(i)}|\bm{z})] - KL(q_{\phi_K}(\bm{z} | \bm{X}_K^{(i)}) || p_\theta(\bm{z}))
\end{align}
$\mathcal{L}_K(\theta, \phi_K ; \bm{X}^{(i)})$ defines the ELBO if only $\bm{X}_K^{(i)}$ is available, but we are interested in the true posterior distribution $p_\theta(\bm{z} | \bm{X}^{(i)})$. To improve readability, we will omit the superscript $(i)$ in the remaining part of this work.

\subsection{Related Work}
\label{subsec:related_work}
In this work, we focus on methods with the aim of modelling a joint latent distribution, instead of translating between modalities \cite{huang2018multimodal,tian2019latent} due to the scalability constraint described in \Cref{sec:intro}.

\textbf{Joint and Conditional Generation.} \cite{Suzuki2016} implemented a multimodal VAE and introduced the idea that the distribution of the unimodal approximation should be close to the multimodal approximation function. \cite{vedantam2017generative} introduced the triple ELBO as an additional improvement. Both define labels as second modality and are not scalable in the number of modalities.

\textbf{Modality-specific Latent Subspaces.} \cite{Hsu2018DisentanglingData,tsai2018learning} both proposed models with modality-specific latent distributions and an additional shared distribution. The former relies on supervision by labels to extract modality-independent factors, while the latter is non-scalable. \cite{daunhawer2020} are also able to show the advantage of modality-specific sub-spaces.

\textbf{Scalability.} More recently, \cite{Kurle2018Multi-SourceInference,Wu2018MultimodalLearning} proposed scalable multimodal generative models for which they achieve scalability by using a Product-of-Experts (PoE) \citep{hinton2002training} as a joint approximation distribution. The PoE allows them to handle missing modalities without requiring separate inference networks for every combination of missing and available data. A PoE is computationally attractive as - for Gaussian-distributed experts - it remains Gaussian distributed which allows the calculation of the KL-divergence in closed form. However, they report problems in optimizing the unimodal variational approximation distributions due to the multiplicative nature of the PoE. To overcome this limitation, \cite{Wu2018MultimodalLearning} introduced a combination of ELBOs which results in the final objective not being an ELBO anymore~\citep{wu2019multimodal}. \cite{shi2019variational} use a Mixture-of-Experts (MoE) as joint approximation function. The additive nature of the MoE facilitates the optimization of the individual experts, but is computationally less efficient as there exists no closed form solution to calculate the KL-divergence. \cite{shi2019variational} need to rely on importance sampling (IS) to achieve the desired quality of samples. IS based VAEs \cite{Burda2015} tend to achieve tight ELBOs for the price of a reduced computational efficiency. Additionally, their model requires $M^2$ passes through the decoder networks which increases the computational cost further.


\section{The multimodal JS-Divergence  model}
\label{sec:methods}
We propose a new multimodal objective (mmJSD) utilizing the Jensen-Shannon divergence. Compared to previous work, this formulation does not need any additional training objectives \cite{Wu2018MultimodalLearning}, supervision \cite{tsai2018learning} or importance sampling \cite{shi2019variational}, while being scalable \cite{Hsu2018DisentanglingData}.
\begin{definition}
\label{thm:mm_objective}
\
\begin{enumerate}
     \item Let $\bm{\pi}$ be the distribution weights: $\bm{\pi} = [\pi_1, \ldots, \pi_{M+1}]$ and $\sum_i \pi_i = 1$.
    \item Let $JS_{\bm{\pi}}^{M+1}$ be the Jensen-Shannon divergence for $M+1$ distributions 
    \begin{flalign}
    \label{eq:jsd}
    JS_{\bm{\pi}}^{M+1} (\{ q_{j}(\bm{z})\}_{j=1}^{M+1}) = \sum_{j=1}^{M+1} \pi_j KL(q_{j}(\bm{z}) |  f_{\mathcal{M}}(\{ q_{\nu}(\bm{z}) \}_{\nu=1}^{M+1})\,.&&
    \end{flalign}
    where the function $f_{\mathcal{M}}$ defines a mixture distribution of its arguments~\citep{lin1991divergence}.
\end{enumerate}

We define a new objective $\widetilde{\mathcal{L}}(\theta, \phi ; \bm{X})$ for learning multimodal, generative models which utilizes the Jensen-Shannon divergence:
\begin{align}
\label{eq:obj_mm_jsd}
    \widetilde{\mathcal{L}}(\theta, \phi ; \bm{X}) := E_{q_{\phi}(\bm{z}|\bm{X})} [\log p_\theta ( \bm{X} | \bm{z})] - JS_{\bm{\pi}}^{M+1} (\{ q_{\phi_{\bm{z}_j}}(\bm{z}|\bm{x}_j) \}_{j=1}^M, p_{\theta}(\bm{z}))
\end{align}
\end{definition}

The JS-divergence for $M+1$ distributions is the extension of the standard JS-divergence for two distributions to an arbitrary number of distributions. It is a weighted sum of KL-divergences between the $M+1$ individual probability distributions $q_{j}(\bm{z})$ and their mixture distribution $f_{\mathcal{M}}$. In the remaining part of this section, we derive the new objective directly from the standard ELBO formulation and prove that it is a lower bound to the marginal log-likelihood $\log p_{\theta}(\bm{X}^{(i)})$.

\subsection{Joint Distribution}
\label{subsec:joint_mixture_distribution}
A MoE is an arithmetic mean function whose additive nature facilitates the optimization of the individual experts compared to a PoE (see \Cref{subsec:related_work}). As there exists no closed form solution for the calculation of the respective KL-divergence, we need to rely on an upper bound to the true divergence using Jensen's inequality \cite{hershey2007approximating} for an efficient calculation (for details please see Appendix B). In a first step towards \Cref{eq:obj_mm_jsd}, we approximate the multimodal ELBO defined in \Cref{eq:elbo_mm} by a sum of KL-terms:
\begin{align}
\label{eq:elbo_mm_approximation}
    \mathcal{L}(\theta, \phi ; \bm{X}) \geq E_{q_{\phi}(\bm{z}|\bm{X})} [\log p_{\theta} ( \bm{X} | \bm{z})] - \sum_{j=1}^M \pi_j KL(q_{\phi_j}(\bm{z} | \bm{x}_j) || p_{\theta}(\bm{z}))
\end{align}
The sum of KL-divergences can be calculated in closed form if prior distribution $p_\theta (\bm{z})$ and unimodal posterior approximations $q_{\phi_j}(\bm{z} | \bm{x}_j)$ are both Gaussian distributed. In the Gaussian case, this lower bound to the ELBO $\mathcal{L}(\theta, \phi ; \bm{X})$ allows the optimization of the ELBO objective in a computationally efficient way. 

\subsection{Dynamic Prior}
In the regularization term in \Cref{eq:elbo_mm_approximation}, although efficiently optimizable, the unimodal approximations $q_{\phi_j}(\bm{z} | \bm{x}_j)$ are only individually compared to the prior, and no joint objective is involved. We propose to incoporate the unimodal posterior approximations into the prior through a function $f$.

\begin{definition}[Multimodal Dynamic Prior]
    \label{thm:mm_dynamic_prior}
    The dynamic prior is defined as a function $f$ of the unimodal approximation functions $\{q_{\phi_\nu}(\bm{z}|\bm{x}_{\nu})\}_{\nu=1}^M$ and a pre-defined distribution $p_\theta(\bm{z})$:
    \begin{equation}
    \label{eq:dyn_prior}
		p_f(\bm{z}|\bm{X}) = f(\{ q_{\phi_\nu}(\bm{z}|\bm{x}_{\nu}) \}_{\nu=1}^M, p_{\theta}(\bm{z}))
	\end{equation}
\end{definition}

The dynamic prior is not a prior distribution in the conventional sense as it does not reflect prior knowledge of the data, but it incorporates the prior knowledge that all modalities share common factors. We therefore call it \emph{prior} due to its role in the ELBO formulation and optimization. As a function of all the unimodal posterior approximations, the dynamic prior extracts the shared information and relates the unimodal approximations to it. With this formulation, the objective is optimized at the same time for a similarity between the function $f$ and the unimodal posterior approximations. For random sampling, the pre-defined prior $p_{\theta}(\bm{z})$ is used.

\subsection{Jensen-Shannon Divergence}
\label{sec:jsd}
Utilizing the dynamic prior $p_f(\bm{z} | \bm{X})$, the sum of KL-divergences in \Cref{eq:elbo_mm_approximation} can be written as JS-divergence (see \Cref{eq:jsd}) if the function $f$ defines a mixture distribution. To remain a valid ELBO, the function $p_f(\bm{z} | \bm{X})$ needs to be a well-defined prior.

\begin{lemma}
\label{thm:welldefined_moe_prior}
    If the function $f$ of the dynamic prior $p_f(\bm{z} | \bm{X})$ defines a mixture distribution of the unimodal approximation distributions $\{q_{\phi_\nu}(\bm{z} | \bm{x}_\nu)\}_{\nu=1}^M$, the resulting dynamic prior $p_{\text{MoE}}(\bm{z} | \bm{X})$ is well-defined.
\end{lemma}
\begin{proof}
    The proof can be found in Appendix B.
\end{proof}
With \Cref{thm:welldefined_moe_prior}, the new multimodal objective $\widetilde{\mathcal{L}}(\theta, \phi ; \bm{X})$ utilizing the Jensen-Shannon divergence (\Cref{thm:mm_objective}) can now be directly derived from the ELBO in \Cref{eq:elbo_mm}.

\begin{lemma}
\label{thm:elbo_mm_jsd}
The multimodal objective $\widetilde{\mathcal{L}}(\theta, \phi ; \bm{X})$ utilizing the Jensen-Shannon divergence defined in \Cref{eq:obj_mm_jsd} is a lower bound to the ELBO in \Cref{eq:elbo_mm}.
\begin{align}
\label{eq:elbo_mm_jsd}
    \mathcal{L}(\theta, \phi ; \bm{X}) \geq \widetilde{\mathcal{L}}(\theta, \phi ; \bm{X})
\end{align}
\end{lemma}
\begin{proof}
    The lower bound to the ELBO in \Cref{eq:elbo_mm_approximation} can be rewritten using the dynamic prior $p_{\text{MoE}}(\bm{z}|\bm{X})$:
    \begin{align}
    \label{eq:elbo_mm_jsd_derivation}
	    \mathcal{L}(\theta, \phi ; \bm{X}) \geq& E_{q_{\phi}(\bm{z}|\bm{X})}[\log p_{\theta}(\bm{X}|\bm{z})] - \sum_{j=1}^M \pi_j KL( q_{\phi_j}(\bm{z} | \bm{x}_j)|| p_{\text{MoE}}(\bm{z}|\bm{X})) \nonumber \\
	    & - \pi_{M+1} KL(p_{\theta}(\bm{z})|| p_{\text{MoE}}(\bm{z}|\bm{X})) \nonumber \\
	    =& E_{q_{\phi}(\bm{z}|\bm{X})}[\log p_{\theta}(\bm{X}|\bm{z})] - JS_{\bm{\pi}}^{M+1} (\{ q_{\phi_j}(\bm{z}|\bm{x}_j) \}_{j=1}^M, p_{\theta}(\bm{z})) \nonumber \\
	    =& \widetilde{\mathcal{L}}(\theta, \phi ; \bm{X})
    \end{align}
    Proving that $\widetilde{\mathcal{L}}(\theta, \phi ; \bm{X})$ is a lower bound to the original ELBO formulation in \Cref{eq:elbo_mm} also proves that it is a lower bound the marginal log-likelihood $\log p_{\theta}(\bm{X}^{(i)})$. This makes the proposed objective an ELBO itself.\footnote{We would like to emphasize that the lower bound in the first line of \Cref{eq:elbo_mm_jsd_derivation} originates from \Cref{eq:elbo_mm_approximation} and not from the introduction of the dynamic prior. }
\end{proof}

The objective in \Cref{eq:obj_mm_jsd} using the JS-divergence is an intuitive extension of the ELBO formulation to the multimodal case as it relates the unimodal to the multimodal approximation functions while providing a more expressive prior \cite{tomczak2017vaebb}.
In addition, it is important to notice that the function $f$ of the dynamic prior $p_f(\bm{z}|\bm{X})$, e.g. an arithmetic mean as in $p_{\text{MoE}}(\bm{z}|\bm{X})$, is not related to the definition of the joint posterior approximation $q_{\phi}(\bm{z}|\bm{X})$. Hence, \Cref{thm:mm_objective} is a special case which follows the definition of the dynamic prior $p_f(\bm{z}|\bm{X})$ as $p_{\text{MoE}}(\bm{z}|\bm{X})$ -- or other abstract mean functions (see \Cref{subsec:jsd_generalized}).

\subsection{Generalized Jensen-Shannon Divergence}
\label{subsec:jsd_generalized}
\cite{Nielsen2019} defines the JS-divergence for the general case of abstract means. This allows to calculate the JS-divergence not only using an arithmetic mean as in the standard formulation, but any mean function.
Abstract means are a suitable class of functions for aggregating information from different distributions while being able to handle missing data \cite{Nielsen2019}.
\begin{definition}
\label{thm:poe_prior}
    The dynamic prior $p_{\text{PoE}}(\bm{z} | \bm{X})$ is defined as the geometric mean of the unimodal posterior approximations $\{q_{\phi_\nu}(\bm{z} | \bm{x}_\nu)\}_{\nu=1}^M$ and the pre-defined distribution $p_\theta(\bm{z})$.
\end{definition}
For Gaussian distributed arguments, the geometric mean is again Gaussian distributed and equivalent to a weighted PoE \cite{hinton2002training}.
The proof that $p_{\text{PoE}}(\bm{z} | \bm{X})$ is a well-defined prior can be found in Appendix B.

Utilizing \Cref{thm:poe_prior}, the JS-divergence in \Cref{eq:obj_mm_jsd} can be calculated in closed form. This allows the optimization of the proposed, multimodal objective $\widetilde{\mathcal{L}}(\theta, \phi ; \bm{X})$ in a computationally efficient way. As the unimodal posterior approximations are directly optimized as well,  $\widetilde{\mathcal{L}}(\theta, \phi ; \bm{X})$ using a PoE-prior also tackles the limitations of previous work outlined in \Cref{subsec:related_work}. Hence, we use a dynamic prior of the form $p_{\text{PoE}}(\bm{z} | \bm{X})$ for our experiments.

\subsection{Modality-specific Latent Subspaces}
\label{sec:mod_spec_subspaces}
We define our latent representations as a combination of modality-specific spaces and a shared, modality-independent space: $\bm{z} = (\bm{S}, \bm{c}) = (\{\bm{s}_j\}_{j=1}^M, \bm{c})$. Every $\bm{x}_j$ is modelled to have its own independent, modality-specific part $\bm{s}_j$. Additionally, we assume a joint content $\bm{c}$ for all $\bm{x}_j \in \bm{X}$ which captures the information that is shared across modalities. $\bm{S}$ and $\bm{c}$ are considered conditionally independent given $\bm{X}$. Different to previous work \cite{bouchacourt2018multi,tsai2018learning,wieser2020inverse}, we empirically show that meaningful representations can be learned in a self-supervised setting by the supervision which is given naturally for multimodal problems. Building on what we derived in \Cref{sec:background,sec:methods}, and the assumptions outlined above, we model the modality-dependent divergence term similarly to the unimodal setting as there is no intermodality relationship associated with them. Applying these assumptions to \Cref{eq:obj_mm_jsd}, it follows (for details, please see Appendix B):
\begin{align}
\label{eq:elbo_mm_jsd_factorized}
	\widetilde{\mathcal{L}}(\theta, \phi ; \bm{X}) =& \sum_{j=1}^M E_{q_{\phi_{\bm{c}}}(\bm{c} | \bm{X})}[E_{q_{\phi_{\bm{s}_j}}(\bm{s}_j | \bm{x}_j)}[\log p_\theta(\bm{x}_j|\bm{s}_j,\bm{c})]] \\
    &- \sum_{j=1}^M D_{KL}(q_{\phi_{\bm{s}_j}}(\bm{s}_j | \bm{x}_j) || p_\theta(\bm{s}_j)) - JS_{\bm{\pi}}^{M+1} (\{ q_{\phi_{\bm{c}_j}}(\bm{c}|\bm{x}_j) \}_{j=1}^M, p_\theta(\bm{c})) \nonumber
\end{align}
The objective in~\Cref{eq:obj_mm_jsd} is split further into two different divergence terms: The JS-divergence is used only for the multimodal latent factors $\bm{c}$, while modality-independent terms $\bm{s}_j$ are part of a sum of KL-divergences. Following the common line in VAE-research, the variational approximation functions $q_{\phi_{\bm{c}_j}}(\bm{c}_j | \bm{x}_j)$ and $q_{\phi_{\bm{s}_j}}(\bm{s_j}|\bm{x}_j)$, as well as the generative models $p_{\theta}(\bm{x}_j | \bm{s}_j, \bm{c})$ are parameterized by neural networks.

\begin{figure*}[ht]
\vskip 0.2in
    \centering
    \begin{subfigure}[b]{0.3\textwidth}
        \centering
        \includegraphics[width=1.0\textwidth]{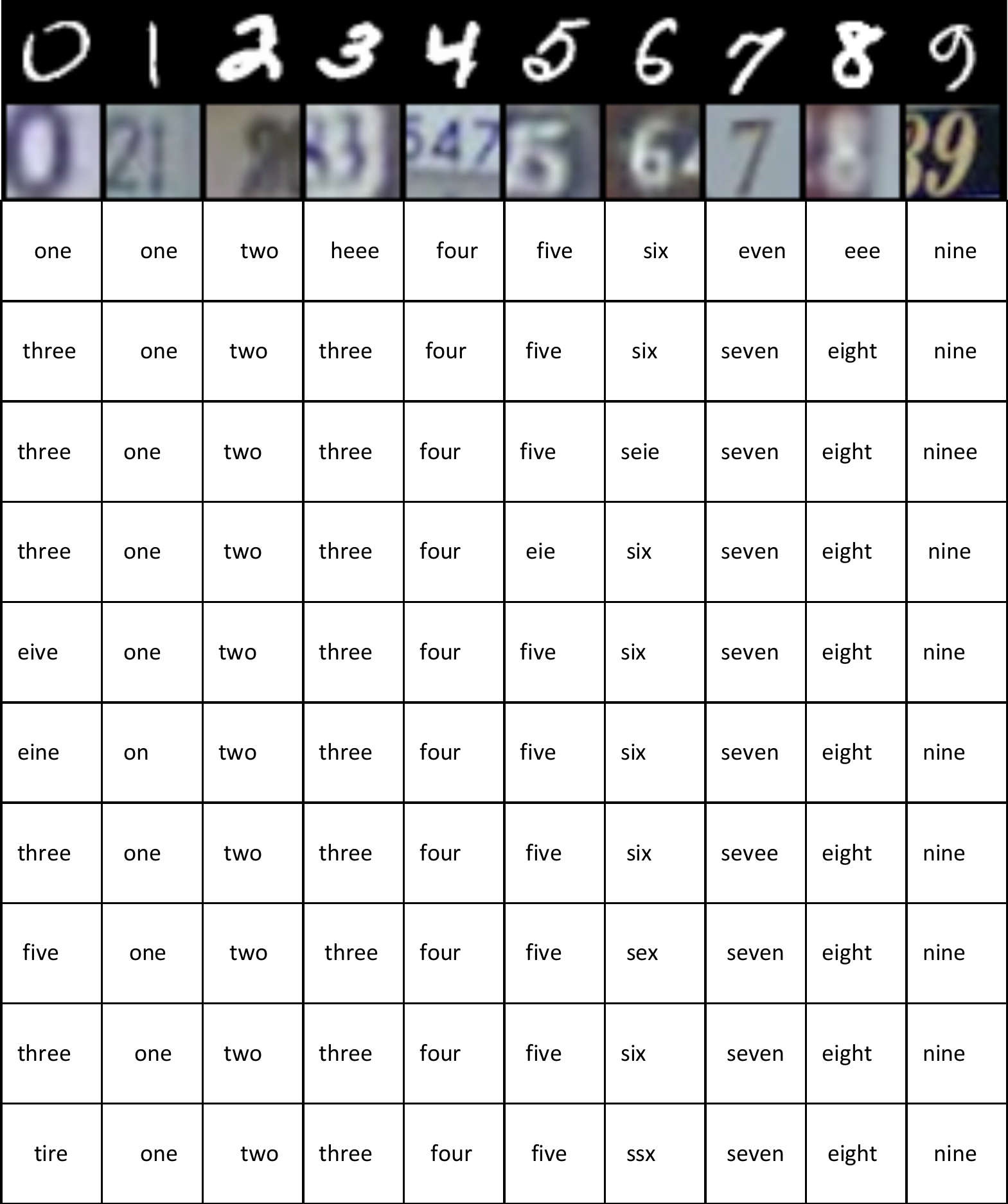}
        \caption{mmJSD (MS): M,S $\rightarrow$ T}
        \label{fig:mst_mmjsd_ms_t}
    \end{subfigure}
    \hfill
    \begin{subfigure}[b]{0.3\textwidth}
        \centering
        \includegraphics[width=1.0\textwidth]{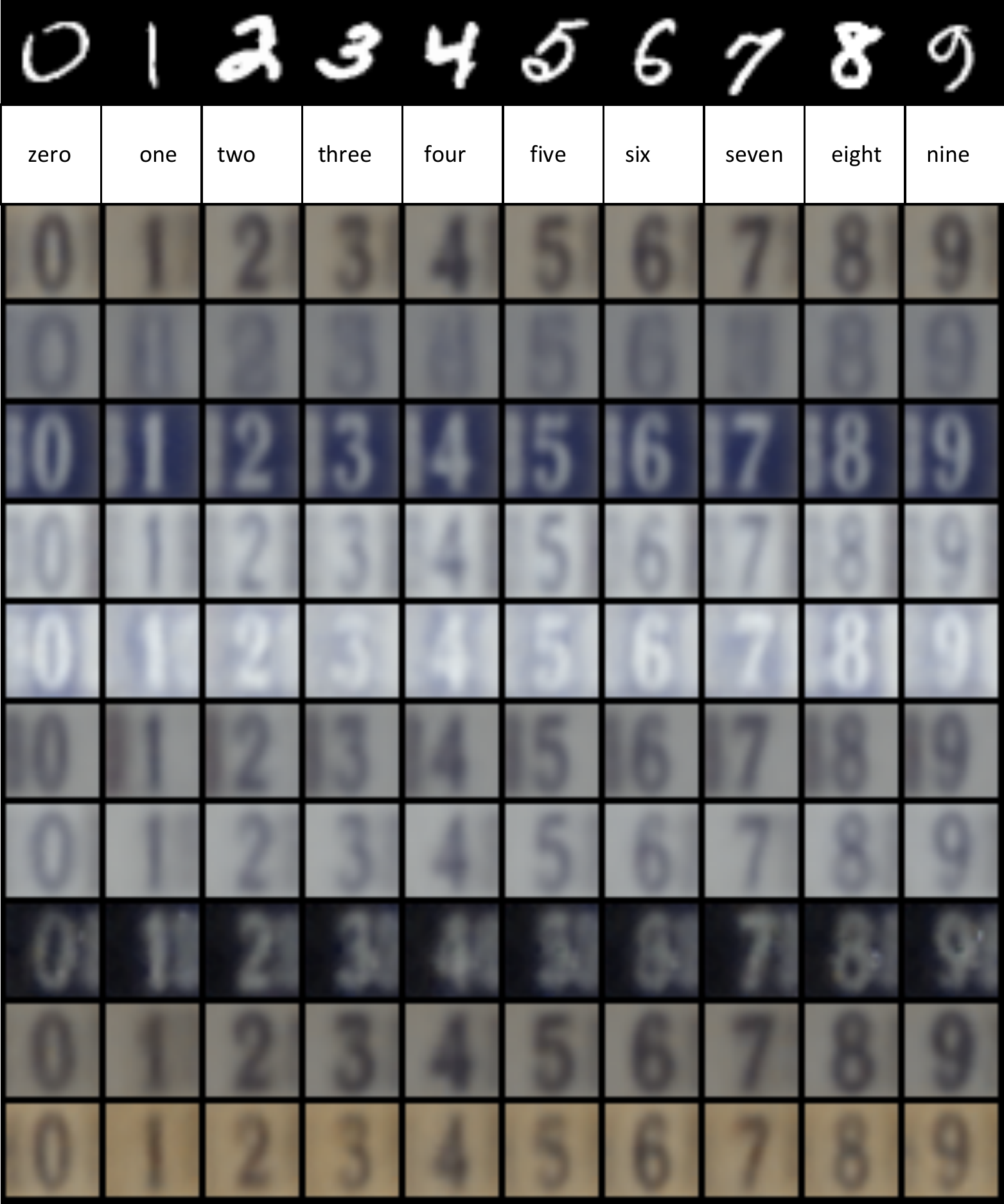}
        \caption{mmJSD (MS): M,T $\rightarrow$ S}
        \label{fig:mst_mmjsd_mt_s}
    \end{subfigure}
    \hfill
    \centering
    \begin{subfigure}[b]{0.3\textwidth}
        \centering
        \includegraphics[width=1.0\textwidth]{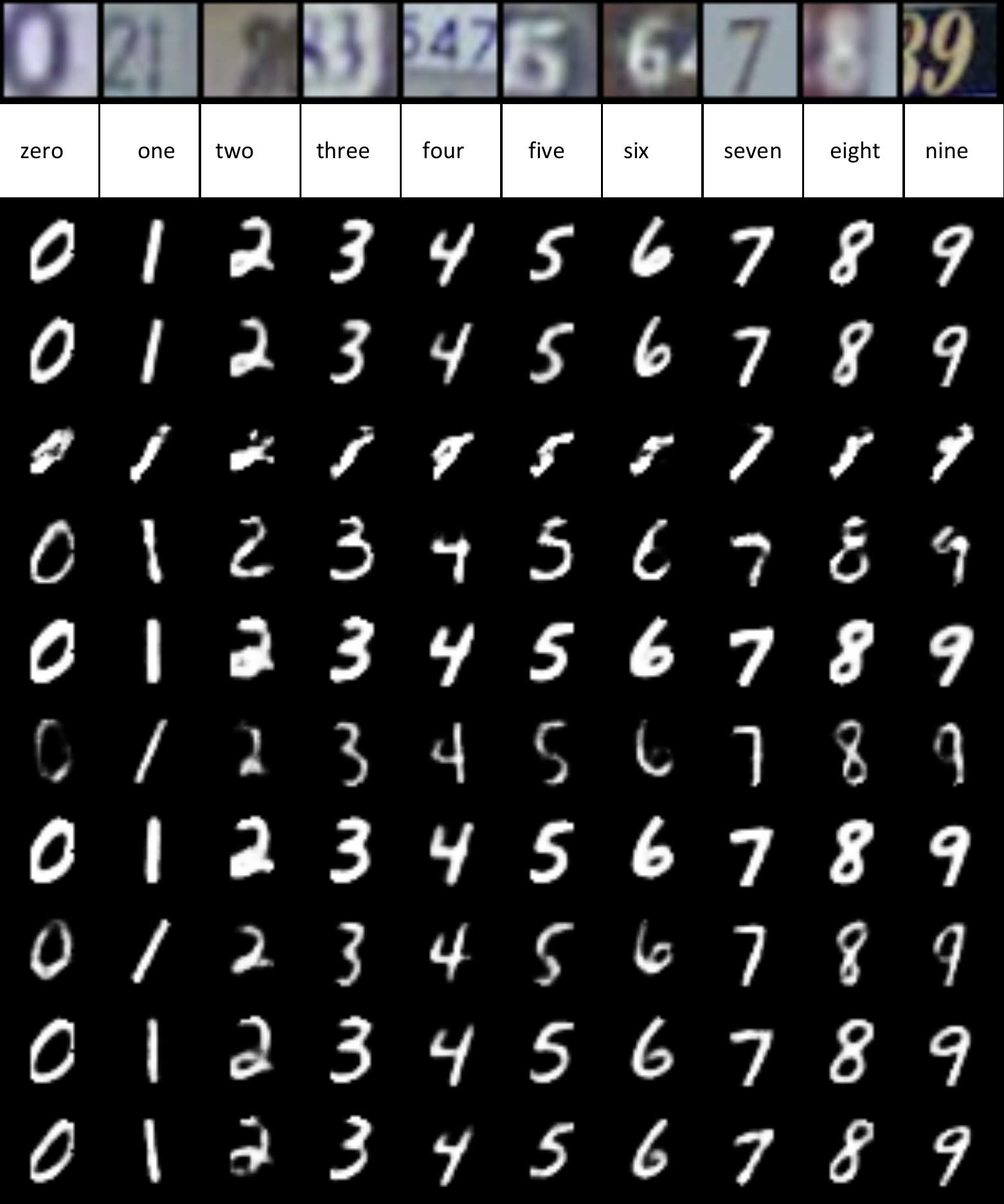}
        \caption{mmJSD (MS): S,T $\rightarrow$ M}
        \label{fig:mst_mmjsd_st_m}
    \end{subfigure}
\caption{Qualitative results for missing data estimation. Each row is generated by a single, random style and the information inferred from the available modalities in the first two rows. This allows for the generation of samples with coherent, random styles across multiple contents (see ~\Cref{tab:mst_representation_classification} for explanation of abbreviations).}
\label{fig:mnistsvhntext_cond_gen}
\vskip -0.2in
\end{figure*}

\section{Experiments \& Results}
\label{sec:experiments}
We carry out experiments on two different datasets\footnote{The code for our experiments can be found \href{https://github.com/thomassutter/mmjsd}{here}.}. For the experiment we use a matching digits dataset consisting of MNIST \cite{lecun-mnisthandwrittendigit-2010} and SVHN \cite{netzer2011reading} images with an additional text modality~\cite{shi2019variational}. This experiment provides empirical evidence on a method's generalizability to more than two modalities. The second experiment is carried out on the challenging CelebA faces dataset \cite{Liu_2015_ICCV} with additional text describing the attributes of the shown face. The CelebA dataset is highly imbalanced regarding the distribution of attributes which poses additional challenges for generative models.

\subsection{Evaluation}
We evaluate the performance of models with respect to the multimodal wish-list introduced in \Cref{sec:intro}. 
To assess the discriminative capabilities of a model, we evaluate the latent representations with respect to the input data's semantic information. We employ a linear classifier on the unimodal and multimodal posterior approximations. To assess the generative performance, we evaluate generated samples according to their quality and coherence~\cite{shi2019variational}. Generation should be coherent across all modalities with respect to shared information. Conditionally generated samples should be coherent with the input data, randomly generated samples with each other. To evaluate the coherence of generated samples, we use a classifier which was trained on the original unimodal training set. If the classifier detects the same attributes in the samples, it is a coherent generation \cite{ravuri2019classification}.
To assess the quality of generated samples, we use the precision and recall metric for generative models~\cite{sajjadi2018assessing} where precision defines the quality and recall the diversity of the generated samples. In addition, we evaluate all models regarding their test set log-likelihoods.

We compare the proposed method to two state-of-the-art models: the MVAE model \cite{Wu2018MultimodalLearning} and the MMVAE model \cite{shi2019variational} described in \Cref{subsec:related_work}. We use the same encoder and decoder networks and the same number of parameters for all methods. Implementation details for all experiments together with a comparison of runtimes can be found in Appendix C.

\subsection{MNIST-SVHN-Text}
\label{subsec:exp_mnistsvhntext}
\begin{table}
\caption{Classification accuracy of the learned latent representations using a linear classifier. We evaluate all subsets of modalities for which we use the following abbreviations: M: MNIST; S: SVHN; T: Text; M,S: MNIST and SVHN; M,T: MNIST and Text; S,T: SVHN and Text; Joint: all modalities. (MS) names the models with modality-specific latent subspaces.}
\label{tab:mst_representation_classification}
\vskip 0.15in
\begin{center}
\begin{small}
\begin{sc}
\begin{tabular}{lccccccc}
    \toprule
    Model &  M & S & T & M,S & M,T & S,T & Joint \\
    \midrule
    MVAE & 0.85 & 0.20 & 0.58 & 0.80 & 0.92 & 0.46 & 0.90 \\
    MMVAE & 0.96 & 0.81 & \textbf{0.99} & 0.89 & 0.97 & 0.90 & 0.93 \\
    \textbf{mmJSD} & 0.97 & 0.82 & \textbf{0.99} & 0.93 & \textbf{0.99} & 0.92 & 0.98 \\
    \midrule
    MVAE (MS) & 0.86 & 0.28 & 0.78 & 0.82 & 0.94 & 0.64 & 0.92 \\
    MMVAE (MS) & 0.96 & 0.81 & \textbf{0.99} & 0.89 & 0.98 & 0.91 & 0.92 \\
    \textbf{mmJSD (MS)} & \textbf{0.98} & \textbf{0.85} & \textbf{0.99} & \textbf{0.94} & 0.98 & \textbf{0.94} & \textbf{0.99} \\
    \bottomrule
\end{tabular}
\end{sc}
\end{small}
\end{center}
\vskip -0.1in
\end{table}
Previous works on scalable, multimodal methods performed no evaluation on more than two modalities\footnote{\cite{Wu2018MultimodalLearning} designed a computer vision study with multiple transformations and a multimodal experiment for the CelebA dataset where every attribute is considered a modality.}.
We use the MNIST-SVHN dataset \cite{shi2019variational} as basis. To this dataset, we add an additional, text-based modality. The texts consist of strings which name the digit in English where the start index of the word is chosen at random to have more diversity in the data. To evaluate the effect of the dynamic prior as well as modality-specific latent subspaces, we first compare models with a single shared latent space. In a second comparison, we add modality-specific subspaces to all models (for these experiments, we add a (MS)-suffix to the model names). This allows us to assess and evaluate the contribution of the dynamic prior as well as modality-specific subspaces. Different subspace sizes are compared in Appendix C. 

\begin{table}
\caption{Classification accuracy of generated samples on MNIST-SVHN-Text. In case of conditional generation, the letter above the horizontal line indicates the modality which is generated based on the different sets of modalities below the horizontal line.}
\label{tab:mst_generation_coherence}
\vskip 0.15in
\begin{center}
\begin{small}
\begin{sc}
\begin{tabular}{lcccccccccc}
\toprule
&  & \multicolumn{3}{c}{M} & \multicolumn{3}{c}{S} & \multicolumn{3}{c}{T} \\
\cmidrule(l){3-5} \cmidrule(l){6-8} \cmidrule(l){9-11}
Model & Random & S & T & S,T & M & T & M,T & M & S & M,S \\
\midrule
MVAE & 0.72 & 0.17 & 0.14 & 0.22 & 0.37 & 0.30 & 0.86 & 0.20 & 0.12 & 0.22 \\
MMVAE & 0.54 & \textbf{0.82} & \textbf{0.99} & 0.91 & 0.32 & 0.30 & 0.31 & 0.96 & \textbf{0.83} & 0.90 \\
\textbf{mmJSD} & 0.60 & \textbf{0.82} & \textbf{0.99} & \textbf{0.95} & 0.37 & 0.36 & 0.48 & \textbf{0.97} & \textbf{0.83} & \textbf{0.92} \\
\midrule
MVAE (MS) & \textbf{0.74} & 0.16 & 0.17 & 0.25 & 0.35 & 0.37 & 0.85 & 0.24 & 0.14 & 0.26 \\
MMVAE (MS) & 0.67 & 0.77 & 0.97 & 0.86 & 0.88 & \textbf{0.93} & 0.90 & 0.82 & 0.70 & 0.76 \\
\textbf{mmJSD (MS)} & 0.66 & 0.80 & 0.97 & 0.93 & \textbf{0.89} & \textbf{0.93} & \textbf{0.92} & 0.92 & 0.79 & 0.86 \\
\bottomrule
\end{tabular}
\end{sc}
\end{small}
\end{center}
\vskip -0.1in
\end{table}

\begin{table}[h]
\caption{Quality of generated samples on MNIST-SVHN-Text. We report the average precision based on the precision-recall metric for generative models (higher is better) for conditionally and randomly generated image data (R: Random Generation).}
\label{tab:mst_prd}
\vskip 0.15in
\begin{center}
\begin{small}
\begin{sc}
\begin{tabular}{lcccccccc}
    \toprule
    & \multicolumn{4}{c}{M} & \multicolumn{4}{c}{S} \\
    \cmidrule(l){2-5}\cmidrule(l){6-9}
    Model & S & T & S,T & R & M & T & M, T & R \\
    \midrule
    MVAE & \textbf{0.62} & 0.62 & 0.58 & \textbf{0.62} & \textbf{0.33} & \textbf{0.34} & \textbf{0.22} & \textbf{0.33} \\
    MMVAE & 0.22 & 0.09 & 0.18 & 0.35 & 0.005 & 0.006 & 0.006 & 0.27 \\
    \textbf{mmJSD} & 0.19 & 0.09 & 0.16 & 0.15 & 0.05 & 0.01 & 0.06 & 0.09 \\
    \midrule
    MVAE (MS) & 0.60 & 0.59 & 0.50 & 0.60 & 0.30 & 0.33 & 0.17 & 0.29 \\
    MMVAE (MS) & \textbf{0.62} & 0.63 & 0.63 & 0.52 & 0.21 & 0.20 & 0.20 & 0.19 \\
    \textbf{mmJSD (MS)} & \textbf{0.62} & \textbf{0.64} & \textbf{0.64} & 0.30 & 0.21 & 0.22 & \textbf{0.22} & 0.17 \\
    \bottomrule
  \end{tabular}
\end{sc}
\end{small}
\end{center}
\vskip -0.1in
\end{table}

\Cref{tab:mst_representation_classification,tab:mst_generation_coherence} demonstrate that the proposed mmJSD objective generalizes better to three modalities than previous work. The difficulty of the MVAE objective in optimizing the unimodal posterior approximation is reflected in the coherence numbers of missing data types and the latent representation classification. Although MMVAE is able to produce good results if only a single data type is given, the model cannot leverage the additional information of multiple available observations. Given multiple modalities, the corresponding performance numbers are the arithmetic mean of their unimodal pendants. The mmJSD model is able to achieve state-of-the-art performance in optimizing the unimodal posteriors as well as outperforming previous work in leveraging multiple modalities thanks to the dynamic prior. The quality of random samples might be affected by the dynamic prior: this needs to be investigated further in future work.
The introduction of modality-specific subspaces increases the coherence of the difficult SVHN modality for MMVAE and mmJSD. More importantly, modality-specific latent spaces improve the quality of the generated samples for all modalities (see \Cref{tab:mst_prd}).~\Cref{fig:mnistsvhntext_cond_gen} shows qualitative results.~\Cref{tab:exp_likelihoods_mnistsvhntext} provides evidence that the high coherence of generated samples of the mmJSD model are not traded off against test set log-likelihoods. It also shows that MVAE is able to learn the statistics of a dataset well, but not to preserve the content in case of missing modalities.

\begin{table}[h]
\caption{Test set log-likelihoods on MNIST-SVHN-Text. We report the log-likelihood of the joint generative model $p_\theta(\bm{X})$ and the log-likelihoods of the joint generative model conditioned on the variational posterior of subsets of modalities $q_{\phi_K}(\bm{z}|\bm{X}_K)$. ($\bm{x}_M$: MNIST; $\bm{x}_S$: SVHN; $\bm{x}_T$: Text; $\bm{X} = (\bm{x}_M, \bm{x}_S,\bm{x}_T)$).}
\label{tab:exp_likelihoods_mnistsvhntext}
\vskip 0.15in
\begin{center}
\begin{small}
\begin{sc}
\begin{tabular}{lccccccc}
    \toprule
    Model & $\bm{X}$ & $\bm{X}|\bm{x}_M$ & $\bm{X}|\bm{x}_S$ & $\bm{X}|\bm{x}_T$ & $\bm{X}|\bm{x}_M,\bm{x}_S$ & $\bm{X}|\bm{x}_M,\bm{x}_T$ & $\bm{X}|\bm{x}_S,\bm{x}_T$ \\
    \midrule
    MVAE & \textbf{-1864} & -2002 & -1936 & -2040 & \textbf{-1881} & -1970 & \textbf{-1908} \\
    MMVAE & -1916 & -2213 & \textbf{-1911} & -2250 & -2062 & -2231 & -2080 \\
    \textbf{mmJSD} & -1961 & -2175 & -1955 & -2249 & -2000 & -2121 & -2004 \\
    \midrule
    MVAE (MS) & -1870 & -1999 & -1937 & -2033 & -1886 & -1971 & -1909 \\
    MMVAE (MS) & -1893 & \textbf{-1982} & -1934 & \textbf{-1995} & -1905 & \textbf{-1958} & -1915 \\
    \textbf{mmJSD (MS)} & -1900 & -1994 & -1944 & -2006 & -1907 & -1968 & -1918 \\
    \bottomrule
\end{tabular}
\end{sc}
\end{small}
\end{center}
\vskip -0.1in
\end{table}

\begin{figure}[h]
\vskip 0.2in
    \centering
    \includegraphics[width=0.9\textwidth]{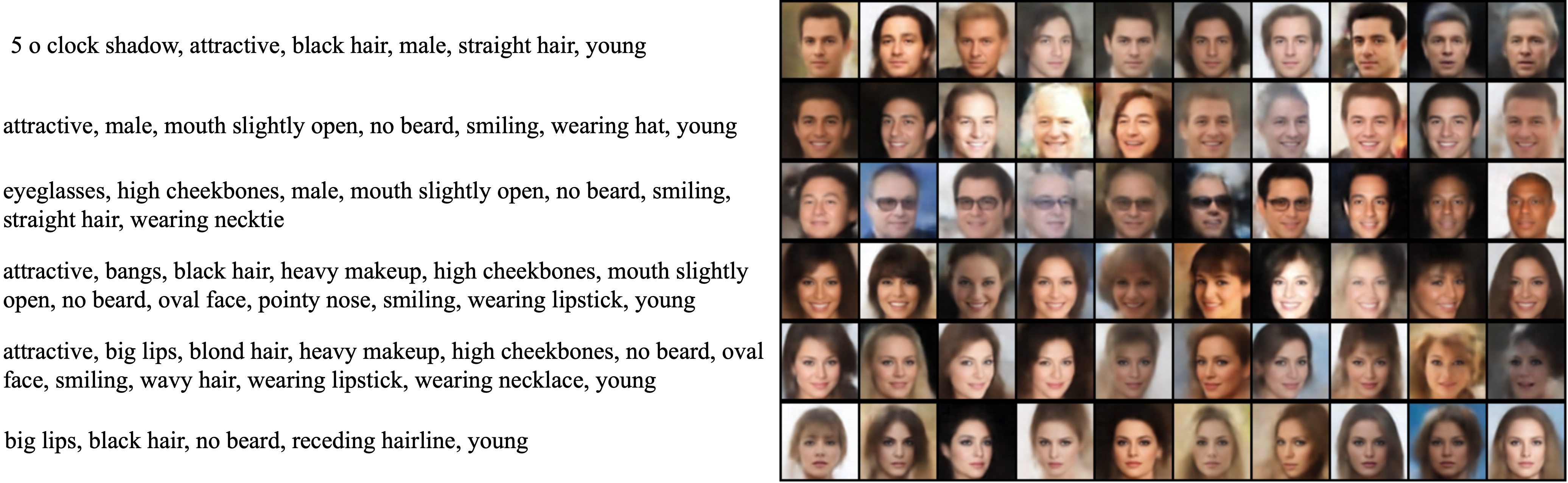}
    \caption{Qualitative Results of CelebA faces which were conditionally generated based on text strings using mmJSD.
    }
\vskip -0.2in
\label{fig:celeba_cond_gen}
\end{figure}

\subsection{Bimodal CelebA}
\label{subsec:exp_celeba}
Every CelebA image is labelled according to 40 attributes. We extend the dataset with an additional text modality describing the face in the image using the labelled attributes. Examples of created strings can be seen in~\Cref{fig:celeba_cond_gen}. Any negative attribute is completely absent in the string. This is different and more difficult to learn than negated attributes as there is no fixed position for a certain attribute in a string which introduces additional variability in the data.
~\Cref{fig:celeba_cond_gen} shows qualitative results for images which are generated conditioned on text samples. Every row of images is based on the text next to it. As the labelled attributes are not capturing all possible variation of a face, we generate 10 images with randomly sampled image-specific information to capture the distribution of information which is not encoded in the shared latent space.
The imbalance of some attributes affects the generative process. Rare and subtle attributes like eyeglasses are difficult to learn while frequent attributes like gender and smiling are well learnt.

~\Cref{tab:celeba_classification} demonstrates the superior performance of the proposed mmJSD objective compared to previous work on the challening bimodal CelebA dataset. The classification results regarding the individual attributes can be found in Appendix C.

\begin{table}[h]
\caption{Classfication results on the bimodal CelebA experiment. For latent representations and conditionally generated samples, we report the mean average precision over all attributes (I: Image; T: Text; Joint: I and T).}
\label{tab:celeba_classification}
\vskip 0.15in
\begin{center}
\begin{small}
\begin{sc}
\begin{tabular}{lccccc}
    \toprule
    & \multicolumn{3}{c}{Latent Representation} & \multicolumn{2}{c}{Generation} \\
    \cmidrule(l){2-4} \cmidrule(l){5-6}
    Model & I & T & Joint & I $\rightarrow$ T & T $\rightarrow$ I \\
    \midrule
    MVAE (MS) & 0.42 & 0.45 & 0.44 & \textbf{0.32} & 0.30  \\
    MMVAE (MS) & 0.43 & 0.45 & 0.42 & 0.30 & 0.36  \\
    \textbf{mmJSD (MS)} & \textbf{0.48} & \textbf{0.59} & \textbf{0.57} & \textbf{0.32} & \textbf{0.42}  \\
    \bottomrule
\end{tabular}
\end{sc}
\end{small}
\end{center}
\vskip -0.1in
\end{table}

\section{Conclusion}
In this work, we propose a novel generative model for learning from multimodal data. Our contributions are fourfold: (i) we formulate a new multimodal objective using a dynamic prior. (ii) We propose to use the JS-divergence for multiple distributions as a divergence measure for multimodal data. This measure enables direct optimization of the unimodal as well as the joint latent approximation functions. (iii) We prove that the proposed mmJSD objective constitutes an ELBO for multiple data types. (iv) With the introduction of modality-specific latent spaces, we show empirically the improvement in quality of generated samples. Additionally, we demonstrate that the proposed method does not need any additional training objectives while  reaching state-of-the-art or superior performance compared to recently proposed, scalable, multimodal generative models. In future work, we would like to further investigate which functions $f$ would serve well as prior function and we will apply our proposed model in the medical domain. 

\section{Broader Impact}

Learning from multiple data types offers many potential applications and opportunities as multiple data types naturally co-occur. We intend to apply our model in the medical domain in future work, and we will focus here on the impact our model might have in the medical application area. Models that are capable of dealing with large-scale  multi-modal data are extremely important in the field of computational medicine and clinical data analysis. 
The recent developments in medical information technology have resulted in an overwhelming amount of multi-modal data available for every single patient.
A patient visit at a hospital may result in tens of thousands of measurements and structured information, including clinical factors, diagnostic imaging, lab tests, genomic and proteomic tests, and hospitals may see thousands of patients each year.
The ultimate aim is to use all this vast information for a medical treatment tailored to the needs of an individual patient.
To turn the vision of precision  medicine into reality, there is an urgent need for the integration of the  multi-modal patient data currently available for improved disease diagnosis, prognosis and therapy outcome prediction. Instead of learning on  one data set exclusively, as for example just on images or just on genetics, the aim is to improve learning and enhance personalized treatment by using as much information as possible for every patient. First steps in this direction have been successful, but so far a major hurdle  has been the huge amount of heterogeneous data with many missing data points which is collected for every patient.  

With this work, we lay the theoretical foundation for the analysis of large-scale multi-modal data. We focus on a self-supervised approach as collecting labels for large datasets of multiple data types is expensive and becomes quickly infeasible with a growing number of modalities. Self-supervised approaches have the potential to overcome the need for excessive labelling and the bias coming from these labels. In this work, we extensively tested the model in controlled environments.  In future work, we will apply our proposed model to medical multi-modal data  with the goal of gaining insights and making predictions about disease phenotypes, disease progression and response to treatment.

\begin{ack}
Thanks to Diane Bouchacourt for providing code and Ri\v{c}ards Marcinkevi\v{c}s for helpful discussions.
ID is supported by the SNSF grant \#200021\_188466.
\end{ack}

\FloatBarrier
\bibliography{references}
\bibliographystyle{abbrvnat}

\newpage

\appendix

In the supplementary material section, we provide additional mathematical derivations, implementation details and results which could not be put in the main paper due to space restrictions.

\section{Theoretical Background}
\label{sec:app_theoretical_background}
The ELBO $\mathcal{L}(\theta, \phi; \bm{X})$ can be derived by reformulating the KL-divergence between the joint posterior approximation function $q_{\phi}(\bm{z} | \bm{X})$ and the true posterior distribution $p_\theta(\bm{z} | \bm{X})$:
\begin{align}
	KL ( q_{\phi}(\bm{z} | \bm{X}) || p_\theta(\bm{z} | \bm{X})) =& \int_{\bm{z}} q_\phi(\bm{z} | \bm{X}) \log \left( \frac{q_\phi(\bm{z} | \bm{X})}{p_\theta(\bm{z} | \bm{X})}\right) dz \nonumber \\
	=& \int_{\bm{z}} q_\phi(\bm{z} | \bm{X}) \log \left( \frac{q_\phi(\bm{z} | \bm{X})p_\theta(\bm{X})}{p_\theta(\bm{X}, \bm{z})}\right) dz \nonumber \\
	=& E_{q_\phi}[\log(q_\phi(\bm{z} | \bm{X})) - \log(p_\theta(\bm{X}, \bm{z}))] + \log(p_\theta(\bm{X}))
\end{align}

It follows:
\begin{align}
    \log p_{\theta}(\bm{X}) = KL ( q_{\phi}(\bm{z} | \bm{X}) || p_\theta(\bm{z} | \bm{X})) - E_{q_\phi}[\log(q_\phi(\bm{z} | \bm{X})) - \log(p_\theta(\bm{X}, \bm{z}))] + \log(p_\theta(\bm{X}))
\end{align}
From the non-negativity of the KL-divergence, it directly follows:
\begin{align}
\label{eq:app_elbo_mm}
	\mathcal{L}(\theta, \phi; \bm{X}) =& E_{q_\phi(\bm{z}|\bm{X})}[\log(p_\theta(\bm{X}|\bm{z})] - KL(q_\phi(\bm{z} | \bm{X}) || p_\theta(\bm{z}))
\end{align}
In the absence of one or multiple data types, we would still like to be able to approximate the true multimodal posterior distribution $p_\theta(\bm{z}|\bm{X})$. However, we are only able to approximate the posterior by a variational function $q_\phi(\bm{z}|\bm{X}_K)$ with $K \leq M$. In addition, for different samples, different modalities might be missing. The derivation of the ELBO formulation changes accordingly:
\begin{align}
	KL ( q_{\phi_K}(\bm{z} | \bm{X}_K) || p_\theta(\bm{z} | \bm{X})) =& \int_{\bm{z}} q_\phi(\bm{z} | \bm{X}_K) \log \left( \frac{q_\phi(\bm{z} | \bm{X}_K)}{p_\theta(\bm{z} | \bm{X})}\right) dz \nonumber \\
	=& \int_{\bm{z}} q_\phi(\bm{z} | \bm{X}_K) \log \left( \frac{q_\phi(\bm{z} | \bm{X}_K)p_\theta(\bm{X})}{p_\theta(\bm{X}, \bm{z})}\right) dz \nonumber \\
	=& E_{q_\phi}[\log(q_\phi(\bm{z} | \bm{X}_K)) - \log(p_\theta(\bm{X}, \bm{z}))] + \log(p_\theta(\bm{X}))
\end{align}
From where it again follows:
\begin{align}
\label{eq:app_elbo_mm_missing}
\mathcal{L}_K(\theta, \phi_K, \bm{X}) = & E_{q_\phi(\bm{z}|\bm{X}_K)}[\log(p_\theta(\bm{X}|\bm{z})] - KL(q_\phi(\bm{z} | \bm{X}_K) || p_\theta(\bm{z}))
\end{align}

\section{Multimodal Jensen-Shannon Divergence Objective}

In this section, we provide the proofs to the Lemmas which were introduced in the main paper. Due to space restrictions, the proofs of these Lemmas had to be moved to the appendix.

\subsection{Upper bound to the KL-divergence of a mixture distribution}
\label{subsec:app_kl_upper_bound}
\begin{lemma}[Joint Approximation Function]
\label{thm:mixture_distribution}
	Under the assumption of $q_\phi(\bm{z} | \{\bm{x}_j \}_{j=1}^M)$ being a mixture model of the unimodal variational posterior approximations $q_{\phi_j}( \bm{z} | \bm{x}_j)$, the KL-divergence of the multimodal variational posterior approximation $q_\phi(\bm{z} | \{\bm{x}_j \}_{j=1}^M)$ is a lower bound for the weighted sum of the KL-divergences of the unimodal variational approximation functions $q_{\phi_j}( \bm{z} | \bm{x}_j)$:
	\begin{align}
	    \label{eq:mixture_approximation}
	    KL(\sum_{j=1}^M \pi_j q_{\phi_j}( \bm{z} | \bm{x}_j) || p_\theta(\bm{z})) \leq \sum_{j=1}^M \pi_j KL(q_{\phi_j}(\bm{z} | \bm{x}_j) || p_\theta(\bm{z}))
	\end{align}
\end{lemma}
\begin{proof}
	Lemma \ref{thm:mixture_distribution} follows directly from the strict convexity of $g(t) = t \log t$.
\end{proof}

\subsection{MoE-Prior}
\label{subsec:app_moe_prior}

\begin{definition}[MoE-Prior]
    \label{thm:moe_prior}
    The prior $p_{MoE}(\bm{z} | \bm{X})$ is defined as follows:
    \begin{align}
    \label{eq:moe_prior}
        p_{MoE}(\bm{z} | \bm{X}) = \sum_{\nu=1}^M \pi_{\nu} q_{\phi_\nu}(\bm{z}|\bm{x}_{\nu}) + \pi_{M+1} p_\theta(\bm{z})
    \end{align}
    where $q_{\phi_\nu}(\bm{z}|\bm{x}_{\nu})$ are again the unimodal approximation functions and $p_\theta(\bm{z})$ is a pre-defined, parameterizable distribution. The mixture weights $\bm{\pi}$ sum to one, i.e. $\sum \pi_{j} = 1$.
\end{definition}

We prove that the MoE-prior $p_{MoE}(\bm{z}|\bm{X})$ is a well-defined prior (see Lemma \ref{thm:welldefined_moe_prior}):
\begin{proof}
    To be a well-defined prior, $p_{MoE}(\bm{z}|\bm{X})$ must satisfy the following condition:
    \begin{equation}
	\label{eq:welldefined_dyn_prior}
		\int p_{MoE}(\bm{z}|\bm{X}) d\bm{z} = 1
	\end{equation}
    Therefore,
    \begin{align}
    \label{eq:proof_moe_prior}
        &\int \left( \sum_{\nu=1}^M \pi_{\nu} q_{\phi_\nu}(\bm{z}|\bm{x}_{\nu}) + \pi_{M+1} p_\theta(\bm{z}) \right ) d\bm{z} \nonumber \\
        &= \sum_{\nu=1}^M \pi_{\nu} \int q_{\phi_\nu}(\bm{z}|\bm{x}_{\nu})d\bm{z} + \pi_{M+1} \int p_\theta(\bm{z}) d\bm{z} \nonumber \\
        &= \sum_{\nu=1}^M \pi_{\nu} + \pi_{M+1} = 1
    \end{align}
    The unimodal approximation functions $q_{\phi_\nu}(\bm{z}|\bm{x}_{\nu})$ as well as the pre-defined distribution $p_\theta(\bm{z})$ are well-defined probability distributions. Hence, $\int q_{\phi_\nu}(\bm{z}|\bm{x}_{\nu})d\bm{z}=1$ for all $q_{\phi_\nu}(\bm{z}|\bm{x}_{\nu})$ and $\int p_\theta(\bm{z}) d\bm{z}=1$. The last line in equation~\ref{eq:proof_moe_prior} follows from the assumptions. Therefore, equation \eqref{eq:moe_prior} is a well-defined prior.
\end{proof}

\subsection{PoE-Prior}
\label{subsec:app_poe_prior}
\begin{lemma}
\label{thm:lemma_poe}
    Under the assumption that all $q_{\phi_\nu}(\bm{z} | \bm{x}_\nu)$ are Gaussian distributed by $\mathcal{N}(\bm{\mu}_\nu(\bm{x}_\nu) , \bm{\sigma}_\nu^2(\bm{x}_\nu) \bm{I})$, $p_{PoE}(\bm{z} | \bm{X})$ is Gaussian distributed:
    \begin{align}
        \label{eq:poe_prior}
        p_{PoE}(\bm{z} | \bm{X}) \sim \mathcal{N}(\bm{\mu}_{GM} , \bm{\sigma}_{GM}^2 \bm{I})
    \end{align}
    where $\bm{\mu}_{GM}$ and $\bm{\sigma}_{GM}^2 \bm{I}$ are defined as follows:
    \begin{align}
    	\bm{\sigma}_{GM}^2 \bm{I} = (\sum_{k=1}^{M+1} \pi_k \bm{\sigma}_k^2 \bm{I})^{-1}, \qquad &
    	\bm{\mu}_{GM} = (\bm{\sigma}_{GM}^2 \bm{I}) \sum_{k=1}^{M+1} \pi_k (\bm{\sigma}_k^2 \bm{I})^{-1} \bm{\mu}_k
    \end{align}
    which makes $p_{PoE}(\bm{z} | \bm{X})$ a well-defined prior.
\end{lemma}
\begin{proof}
    As $p_{PoE}(\bm{z} | \bm{X})$ is Gaussian distributed, it follows immediately that $p_{PoE}(\bm{z} | \bm{X})$ is a well-defined dynamic prior.
\end{proof}
\subsection{Factorization of Representations}
\label{sec:app_factorization_rep}
We mostly base our derivation of factorized representations on the paper by \citet{bouchacourt2018multi}. \citet{tsai2018learning} and \citet{Hsu2018DisentanglingData} used a similar idea. A set $\bm{X}$ of modalities can be seen as group and analogous every modality as a member of a group. We model every $\bm{x}_j$ to have its own modality-specific latent code $\bm{s}_j \in \bm{S}$.
\begin{equation}
\label{eq:app_style}
    \bm{S} = (\bm{s}_j, \forall \bm{x}_j \in \bm{X})
\end{equation}
From Equation \eqref{eq:app_style}, we see that $\bm{S}$ is the collection of all modality-specific latent variables for the set $\bm{X}$. Contrary to this, the modality-invariant latent code $\bm{c}$ is shared between all modalities $\bm{x}_j$ of the set $\bm{X}$.
Also like \citet{bouchacourt2018multi}, we model the variational approximation function $q_{\phi}(\bm{S}, \bm{c})$ to be conditionally independent given $\bm{X}$, i.e.:
\begin{equation}
    \label{eq:app_var_approx_c_s}
    q_{\phi}(\bm{S}, \bm{c}) = q_{\phi_{\bm{S}}}(\bm{S} | \bm{X}) q_{\phi_{\bm{c}}}(\bm{c} | \bm{X})
\end{equation}
From the assumptions it is clear that $q_{\phi_S}$ factorizes:
\begin{equation}
    \label{eq:app_var_approx_style}
    q_{\phi_{\bm{S}}}(\bm{S} | \bm{X}) = \prod_{j=1}^{M} q_{\phi_{\bm{s}_j}}(\bm{s}_j | \bm{x}_j)
\end{equation}
From Equation \eqref{eq:app_var_approx_style} and the fact that the multimodal relationships are only modelled by the latent factor $\bm{c}$, it is reasonable to only apply the mmJSD objective to $\bm{c}$.
It follows:
\begin{align}
    \label{eq:app_factorized_elbo_deri}
    \mathcal{L}(\theta, \phi; \bm{X}) =& E_{q_\phi(\bm{z}|\bm{X})}[\log p_\theta(\bm{X}|\bm{z})] - KL(q_\phi(\bm{z} | \bm{X}) || p_\theta(\bm{z})) \nonumber \\
    =& E_{q_\phi(\bm{S},\bm{c}|\bm{X})}[\log p_\theta(\bm{X}|\bm{S},\bm{c})] - KL(q_\phi(\bm{S},\bm{c} | \bm{X}) || p_\theta(\bm{S},\bm{c})) \nonumber \\
    =& E_{q_\phi(\bm{S},\bm{c}|\bm{X})}[\log p_\theta(\bm{X}|\bm{S},\bm{c})] - KL(q_{\phi_{\bm{S}}}(\bm{S} | \bm{X}) || p_\theta(\bm{S})) - KL(q_{\phi_{\bm{c}}}(\bm{c} | \bm{X}) || p_f(\bm{c})) \nonumber \\
    =& E_{q_\phi(\bm{S},\bm{c}|\bm{X})}[\log p_\theta(\bm{X}|\bm{S},\bm{c})] - \sum_{j=1}^M KL(q_{\phi_{\bm{s}_j}}(\bm{s}_j | \bm{x}_j) || p_\theta(\bm{s}_j)) - KL(q_{\phi_{\bm{c}}}(\bm{c} | \bm{X}) || p_f(\bm{c})) \nonumber \\
\end{align}
In Equation \eqref{eq:app_factorized_elbo_deri}, we can rewrite the KL-divergence which includes $\bm{c}$ using the multimodal dynamic prior and the JS-divergence for multiple distributions:
\begin{align}
    \label{eq:app_factorized_elbo_jsd}
    \widetilde{\mathcal{L}}(\theta, \phi; \bm{X}) =& E_{q_\phi(\bm{S},\bm{c}|\bm{X})}[\log p_\theta(\bm{X}|\bm{S},\bm{c})]- \sum_{j=1}^M KL(q_{\phi_{\bm{s}_j}}(\bm{s}_j | \bm{x}_j) || p_\theta(\bm{s}_j)) \nonumber \\
    &- JS_{\bm{\pi}}^{M+1}(\{ q_{\phi_{\bm{c}_j}}(\bm{c}|\bm{x}_j) \}_{j=1}^M, p_\theta(\bm{c}))
\end{align}

The expectation over $q_{\phi}(\bm{S}, \bm{c} | \bm{X})$ can be rewritten as a concatenation of expectations over $q_{\phi_{\bm{c}}}(\bm{c} | \bm{X})$ and $q_{\phi_{\bm{s}_j}}(\bm{s}_j | \bm{x}_j)$:
\begin{align}
    \label{eq:app_exp_c_s}
    E_{q_\phi(\bm{S},\bm{c}|\bm{X})}[\log p_\theta(\bm{X}|\bm{S},\bm{c})] &= \int_{\bm{c}} \int_{\bm{S}} q_\phi(\bm{S}, \bm{c} | \bm{X}) \log p_\theta(\bm{X}|\bm{S},\bm{c}) d\bm{S} d\bm{c} \nonumber \\
    &= \int_{\bm{c}} q_{\phi_{\bm{c}}}(\bm{c} | \bm{X}) \int_{\bm{S}} q_{\phi_{\bm{S}}}(\bm{S} | \bm{X}) \log p_{\theta}(\bm{X}|\bm{S},\bm{c}) d\bm{S} d\bm{c} \nonumber \\
    &=\int_{\bm{c}} q_{\phi_{\bm{c}}}(\bm{c} | \bm{X}) \sum_{j=1}^M \int_{\bm{s}_j} q_{\phi_{\bm{s}_j}}(\bm{s}_j | \bm{x}_j) \log p_{\theta}(\bm{x}_j|\bm{s}_j,\bm{c}) d\bm{s}_j d\bm{c} \nonumber \\
    &= \sum_{j=1}^M \int_{\bm{c}} q_{\phi_{\bm{c}}}(\bm{c} | \bm{X}) \int_{\bm{s}_j} q_{\phi_{\bm{s}_j}}(\bm{s}_j | \bm{x}_j) \log p_\theta(\bm{x}_j|\bm{s}_j,\bm{c}) d\bm{s}_j d\bm{c} \nonumber \\
    &= \sum_{j=1}^M E_{q_{\phi_{\bm{c}}}(\bm{c} | \bm{X})}[E_{q_{\phi_{\bm{s}_j}}(\bm{s}_j | \bm{x}_j)}[\log p_\theta(\bm{x}_j|\bm{s}_j,\bm{c})]]
\end{align}
From Equation \eqref{eq:app_exp_c_s}, the final form of $\widetilde{\mathcal{L}}(\theta, \phi; \bm{X})$ follows directly:
\begin{align}
    \label{eq:app_factorized_elbo}
    \widetilde{\mathcal{L}}(\theta, \phi; \bm{X}) =& \sum_{j=1}^M E_{q_{\phi_{\bm{c}}}(\bm{c} | \bm{X})}[E_{q_{\phi_{\bm{s}_j}}(\bm{s}_j | \bm{x}_j)}[\log p_\theta(\bm{x}_j|\bm{s}_j,\bm{c})]] \nonumber \\
    &- JS_{\bm{\pi}}^{M+1} (\{ q_{\phi_{\bm{c}_j}}(\bm{c}|\bm{x}_j) \}_{j=1}^M, p_\theta(\bm{c})) - \sum_{j=1}^M KL(q_{\phi_{\bm{s}_j}}(\bm{s}_j | \bm{x}_j) || p_\theta(\bm{s}_j))
\end{align}

\subsection{JS-divergence as intermodality divergence}
Utilizing the JS-divergence as regularization term as proposed in this work has multiple effects on the training procedure. The first is the introduction of the dynamic prior as described in the main paper. A second effect is the minimization of the intermodality-divergence. The intermodality-divergence is the difference of the posterior approximations between modalities. For a coherent generation, the posterior approximations of all modalities should be similar such that - if only a single modality is given - the decoders of the missing data types are able to generate coherent samples.
Using the JS-divergence as regularization term keeps the unimodal posterior approximations similar to its mixture distribution. This can be compared to minimizing the divergence between the unimodal distributions and its mixture which again can be seen as an efficient approximation of minimizing the $M^2$ pairwise unimodal divergences, the intermodality-divergences.
\citet{Wu2018MultimodalLearning} report problems in optimizing the unimodal posterior approximations. These problems lead to diverging posterior approximations which again results in bad coherence for missing data generation. Diverging posterior approximations cannot be handled by the decoders of the missing modality.

\section{Experiments}
In this section we describe the architecture and implementation details of the different experiments. Additionally, we show more results and ablation studies. For the calculation of inception-features we use code provided by~\citet{Seitzer2020FID}.

\subsection{Evaluation}
First we describe the architectures and models used for evaluating classification accuracies.
\subsubsection{Latent Representations}
To evaluate the learned latent representations, we use a simple logistic regression classifier without any regularization. We use a predefined model by scikit-learn \cite{pedregosa2011scikit}. Every linear classifier is trained on a single batch of latent representations. For simplicity, we always take the last batch of the training set to train the classifier. The trained linear classifier is then used to evaluate the latent representations of all samples in the test set.
\subsubsection{Generated Samples}
To evaluate generated samples regarding their content coherence, we classify them according to the attributes of the dataset. In case of missing data, the estimated data types must coincide with the available ones according to the attributes present in the available data types.
In  case of random generation, generated samples of all modalities must be coincide with each other.
To evaluate the coherence of generated samples, classifiers are trained for every modality. If the detected attributes for all involved modalities are the same, the generated samples are called coherent.
For all modalities, classifiers are trained on the original, unimodal training set. The architectures of all used classifiers can be seen in Tables \ref{tab:mnistsvhn_image_classifier} to \ref{tab:celeba_classifiers}.
\begin{table*}
    \caption{Layers for MNIST and SVHN classifiers. For MNIST and SVHN, every convolutional layer is followed by a ReLU activation function. For SVHN, every convolutional layer is followed by a dropout layer (dropout probability = 0.5). Then, batchnorm is applied followed by a ReLU activation function. The output activation is a sigmoid function for both classifiers. Specifications (Spec.) name kernel size, stride, padding and dilation.}
    \label{tab:mnistsvhn_image_classifier}
    \vskip 0.15in
    \begin{center}
    \begin{small}
    \begin{tabular}{llccc|llccc}
        \toprule
        \multicolumn{5}{c}{MNIST} & \multicolumn{5}{c}{SVHN} \\
        Layer & Type & \#F. In & \#F. Out & Spec. & Layer & Type & \#F. In & \#F. Out & Spec. \\
        \midrule
        1 & conv$_{2d}$ & 1 & 32 & (4, 2, 1, 1) & 1 & conv$_{2d}$ & 1 & 32 & (4, 2, 1, 1)\\
        2 & conv$_{2d}$ & 32 & 64 & (4, 2, 1, 1) & 2 & conv$_{2d}$ & 32 & 64 & (4, 2, 1, 1)\\
        3 & conv$_{2d}$ & 64 & 128 & (4, 2, 1, 1) & 3 & conv$_{2d}$ & 64 & 64 & (4, 2, 1, 1)\\
        4 & linear & 128 & 10 & & 4 & conv$_{2d}$ & 64 & 128 & (4, 2, 0, 1)\\
        &&&&& 5 & linear & 128 & 10 & \\
        \bottomrule
    \end{tabular}
    \end{small}
    \end{center}
    \vskip -0.1in
\end{table*}
\begin{table*}
    \caption{Layers for the Text classifier for MNIST-SVHN-Text. The text classifier consists of residual layers as described by \citet{he2016deep} for 1d-convolutions. The output activation is a sigmoid function. Specifications (Spec.) name kernel size, stride, padding and dilation.}
    \label{tab:mnistsvhntext_text_classifier}
    \vskip 0.15in
    \begin{center}
    \begin{small}
    \begin{tabular}{llccc}
        \toprule
        Layer & Type & \#F. In & \#F. Out & Spec. \\
        \midrule
        1 & conv$_{1d}$ & 71 & 128 & (1, 1, 1, 1) \\
        2 & residual$_{1d}$ & 128 & 192 & (4, 2, 1, 1)\\
        3 & residual$_{1d}$ & 192 & 256 & (4, 2, 1, 1)\\
        4 & residual$_{1d}$ & 256 & 256 & (4, 2, 1, 1)\\
        5 & residual$_{1d}$ & 256 & 128 & (4, 2, 0, 1)\\
        6 & linear & 128 & 10 & \\
        \bottomrule
    \end{tabular}
    \end{small}
    \end{center}
    \vskip -0.1in
\end{table*}
\begin{table*}
    \caption{CelebA Classifiers. The image classifier consists of residual layers as described by \citet{he2016deep} followed by a linear layer which maps to 40 output neurons representing the 40 attributes. The text classifier also uses residual layers, but for 1d-convolutions. The output activation is a sigmoid function for both classifiers. Specifications (Spec.) name kernel size, stride, padding and dilation.}
    \label{tab:celeba_classifiers}
    \vskip 0.15in
    \begin{center}
    \begin{small}
    \begin{tabular}{llccc|llccc}
        \toprule
        \multicolumn{5}{c}{Image} & \multicolumn{5}{c}{Text} \\
        Layer & Type & \#F. In & \#F. Out & Spec. & Layer & Type & \#F. In & \#F. Out & Spec. \\
        \midrule
        1 & conv$_{2d}$ & 3 & 128 & (3, 2, 1, 1) & 1 & conv$_{1d}$ & 71 & 128 & (3, 2, 1, 1) \\
        2 & res$_{2d}$ & 128 & 256 & (4, 2, 1, 1) & 2 & res$_{1d}$ & 128 & 256 & (4, 2, 1, 1)\\
        3 & res$_{2d}$ & 256 & 384 & (4, 2, 1, 1) & 3 & res$_{1d}$ & 256 & 384 & (4, 2, 1, 1)\\
        4 & res$_{2d}$ & 384 & 512 & (4, 2, 1, 1) & 4 & res$_{1d}$ & 384 & 512 & (4, 2, 1, 1)\\
        5 & res$_{2d}$ & 512 & 640 & (4, 2, 0, 1) & 5 & res$_{1d}$ & 512 & 640 & (4, 2, 1, 1)\\
        6 & linear & 640 & 40 &  & 6 & residual$_{1d}$ & 640 & 768 & (4, 2, 1, 1)\\
        & & &  &  &  7 & residual$_{1d}$ & 768 & 896 & (4, 2, 0, 1)\\
        & & &  &  &  8 & linear & 896 & 40 & \\
        \bottomrule
    \end{tabular}
    \end{small}
    \end{center}
    \vskip -0.1in
\end{table*}

\subsection{MNIST-SVHN-Text}
\label{subsection:app_mst}
\subsubsection{Text Modality}
To have an additional modality, we generate text from labels. As a single word is quite easy to learn, we create strings of length 8 where everything is a blank space except the digit-word. The starting position of the word is chosen randomly to increase the difficulty of the learning task. Some example strings can be seen in Table \ref{tab:mst_string_examples}.

\begin{table*}
    \caption{Example strings to create an additional text modality for the MNIST-SVHN-Text dataset. This results in triples of texts and two different image modalities.}
    \centering
    \smallskip
        \begin{tabular}{l}
            \qquad \qquad \qquad six \qquad\\
            \qquad \qquad eight \qquad \qquad\\
            three \qquad \qquad \qquad \qquad\\
            \qquad five \qquad \qquad \qquad\\
            \qquad \qquad nine \qquad \qquad \\
            \qquad zero \qquad \qquad \qquad\\
            \qquad \qquad \qquad \qquad four\\
            \qquad \qquad \qquad three \qquad\\
            seven \qquad \qquad \qquad \qquad\\
            \qquad five \qquad \qquad \qquad \\
        \end{tabular}
    \label{tab:mst_string_examples}
\end{table*}

\begin{table}
    \caption{MIST: Encoder and Decoder Layers. Every layer is followed by ReLU activation function. Layers 3a and 3b of the encoder are needed to map to $\mu$ and $\sigma^2\bm{I}$ of the approximate posterior distribution.}
    \label{tab:mst_mnist_layers}
    \vskip 0.15in
    \begin{center}
    \begin{small}
    \begin{tabular}{llcc|llcc}
        \toprule
        \multicolumn{4}{c}{Encoder} & \multicolumn{4}{c}{Decoder} \\
        Layer & Type & \# Features In & \# Features Out & Layer & Type & \# Features In & \# Features Out \\
        \midrule
        1 & linear & 784 & 400 & 1 & linear & 20 & 400 \\
        2a & linear & 400 & 20 & 2 & linear & 400 & 784 \\
        2b & linear & 400 & 20 & & & & \\
        \bottomrule
    \end{tabular}
    \end{small}
    \end{center}
    \vskip -0.1in
\end{table}
\begin{table*}
    \caption{SVHN: Encoder and Decoder Layers. The specifications name kernel size, stride, padding and dilation. All layers are followed by a ReLU activation function.}
    \label{tab:mst_svhn_layers}
    \vskip 0.15in
    \begin{center}
    \begin{small}
    \begin{tabular}{llccc|llccc}
        \toprule
        \multicolumn{5}{c}{Encoder} & \multicolumn{5}{c}{Decoder} \\
        Layer & Type & \#F. In & \#F. Out & Spec. & Layer & Type & \#F. In & \#F. Out & Spec. \\
        \midrule
        1 & conv$_{2d}$ & 3 & 32 & (4, 2, 1, 1) & 1 & linear & 20 & 128 & \\
        2 & conv$_{2d}$ & 32 & 64 & (4, 2, 1, 1) & 2 & conv$_{2d}^T$ & 128 & 64 & (4, 2, 0, 1) \\
        3 & conv$_{2d}$ & 64 & 64 & (4, 2, 1, 1) & 3 & conv$_{2d}^T$ & 64 & 64 & (4, 2, 1, 1) \\
        4 & conv$_{2d}$ & 64 & 128 & (4, 2, 0, 1) & 4 & conv$_{2d}^T$ & 64 & 32 & (4, 2, 1, 1) \\
        5a & linear & 128 & 20 & & 5 & conv$_{2d}^T$ & 32 & 3 & (4, 2, 1, 1) \\
        5b & linear & 128 & 20 & \\
        \bottomrule
    \end{tabular}
    \end{small}
    \end{center}
    \vskip -0.1in
\end{table*}

\begin{table*}
    \caption{Text for MNIST-SVHN-Text: Encoder and Decoder Layers. The specifications name kernel size, stride, padding and dilation. All layers are followed by a ReLU activation function.}
    \label{tab:mst_text_layers}
    \vskip 0.15in
    \begin{center}
    \begin{small}
    \begin{tabular}{llccc|llccc}
        \toprule
        \multicolumn{5}{c}{Encoder} & \multicolumn{5}{c}{Decoder} \\
        Layer & Type & \#F. In & \#F. Out & Spec. & Layer & Type & \#F. In & \#F. Out & Spec. \\
        \midrule
        1 & conv$_{1d}$ & 71 & 128 & (1, 1, 0, 1) & 1 & linear & 20 & 128 & \\
        2 & conv$_{1d}$ & 128 & 128 & (4, 2, 1, 1) & 2 & conv$_{1d}^T$ & 128 & 128 & (4, 1, 0, 1) \\
        3 & conv$_{1d}$ & 128 & 128 & (4, 2, 0, 1) & 3 & conv$_{1d}^T$ & 128 & 128 & (4, 2, 1, 1) \\
        4a & linear & 128 & 20 & & 4 & conv$_{1d}^T$ & 128 & 71 & (1, 1, 0, 1) \\
        4b & linear & 128 & 20 & & & & & & \\
        \bottomrule
    \end{tabular}
    \end{small}
    \end{center}
    \vskip -0.1in
\end{table*}

\subsubsection{Implementation Details}
For MNIST and SVHN, we use the network architectures also utilized by \cite{shi2019variational} (see Table \ref{tab:mst_mnist_layers} and Table \ref{tab:mst_svhn_layers}). The network architecture used for the Text modality is described in Table \ref{tab:mst_text_layers}.
For all encoders, the last layers named a and b are needed to map to $\mu$ and $\sigma^2\bm{I}$ of the posterior distribution. In case of modality-specific sub-spaces, there are four last layers to map to $\mu_s$ and $\sigma_s^2\bm{I}$ and $\mu_c$ and $\sigma_c^2\bm{I}$.

To enable a joint latent space, all modalities are mapped to have a 20 dimensional latent space (like in \citet{shi2019variational}). For a latent space with modality-specific and -independent sub-spaces, this restriction is not needed anymore. Only the modality-invariant sub-spaces of all data types must have the same number of latent dimensions. Nevertheless, we create modality-specific sub-spaces of the same size for all modalities. For the results reported in the main text, we set it to 4. To have an equal number of parameters as in the experiment with only a shared latent space, we set the shared latent space to 16 dimensions. This allows for a fair comparison between the two variants regarding the capacity of the latent space. See \cref{subsec:mst_ms} and Figure \ref{fig:mst_abl_ms_dim} for a detailed comparison regarding the size of the modality specific-subspaces. Modality-specific sub-spaces are a possibility to account for the difficulty of every data type.

The image modalities are modelled with a Laplace likelihood and the text modality is modelled with a categorical likelihood. The likelihood-scaling is done according to the data size of every modality. The weight of the largest data type, i.e. SVHN, is set to 1.0. The weight for MNIST is given by $size(SVHN)/size(MNIST)$ and the text weight by $size(SVHN)/size(Text)$. This scaling scheme stays the same for all experiments. The weight of the unimodal posteriors are equally weighted to form the joint distribution. This is true for MMVAE and mmJSD. For MVAE, the posteriors are weighted according to the inverse of their variance.
For mmJSD, all modalities and the pre-defined distribution are weighted $0.25$. We keep this for all experiments reported in the main paper. See \cref{subsec:mst_weight_dp} and Figure \ref{fig:mst_abl_dp_weight} for a more detailed analysis of distribution weights. 

For all experiments, we set $\beta$ to 5.0. For all experiments with modality-specific subspaces, the $\beta$ for the modality-specific subspaces is set equal to the number of modalities, i.e. 3. Additionally, the $\beta$ for the text modality is set to 5.0, for the other 2 modalities it is set to 1.0. The evaluation of different $\beta$-values shows the stability of the model according to this hyper-parameter (see Figure~\ref{fig:mst_abl_beta}).

All unimodal posterior approximations are assumed to be Gaussian distributed  $\mathcal{N}(\bm{\mu}_\nu(\bm{x}_\nu) , \bm{\sigma}_\nu^2(\bm{x}_\nu) \bm{I})$, as well as the pre-defined distribution $p_\theta(\bm{z})$ which is defined as $\mathcal{N}(\bm{0}, \bm{I})$.

For training, we use a batch size of 256 and a starting learning rate of 0.001 together with an ADAM optimizer \cite{kingma2014adam}. We pair every MNIST image with 20 SVHN images which increases the dataset size by a factor of 20. We train our models for 50 epochs in case of a shared latent space only. In case of modality-specific subspaces we train the models for 100 epochs. This is the same for all methods.

\begin{figure*}
    \centering
    \begin{subfigure}[b]{0.45\textwidth}
        \centering
        \includegraphics[width=1.0\textwidth]{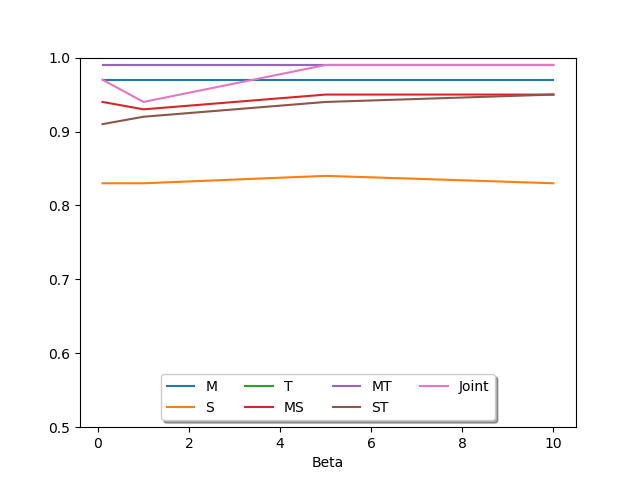}
        \caption{Latent Representation Classification}
        \label{fig:mst_abl_beta_rep}
    \end{subfigure}
    \begin{subfigure}[b]{0.45\textwidth}
        \centering
        \includegraphics[width=1.0\textwidth]{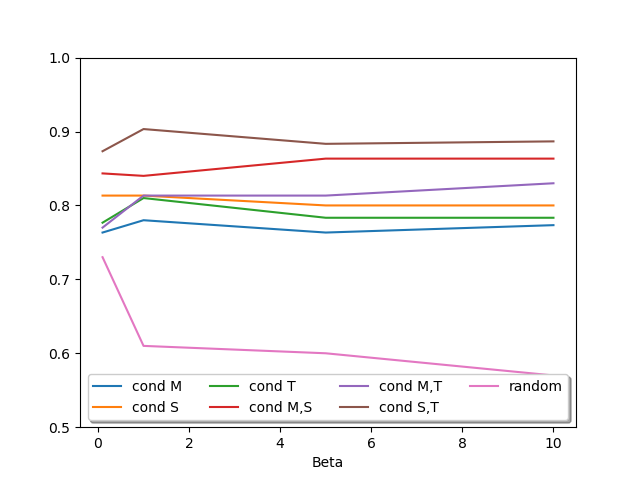}
        \caption{Generation Coherence}
        \label{fig:mst_abl_beta_coherence}
    \end{subfigure}
    \begin{subfigure}[b]{0.45\textwidth}
        \centering
        \includegraphics[width=1.0\textwidth]{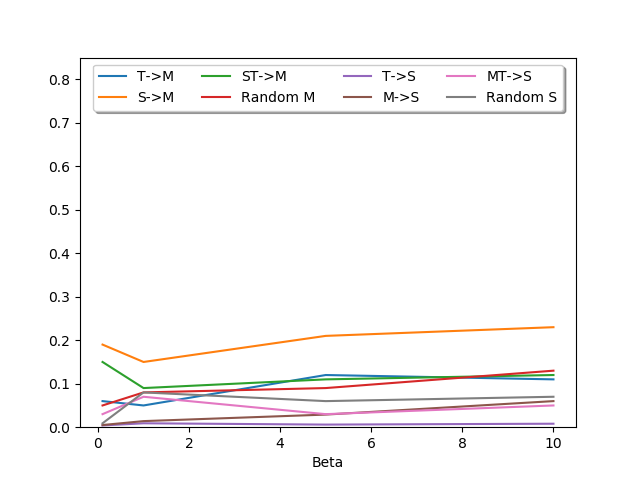}  
        \caption{Quality of Samples}
        \label{fig:mst_abl_beta_prd}
    \end{subfigure}
    \centering
    \caption{Comparison of different $\beta$ values with respect to generation coherence, quality of latent representations (measured in accuracy) and quality of generated samples (measured in precision-recall for generative models).}
    \label{fig:mst_abl_beta}
\end{figure*}

\subsubsection{Qualitative Results}
Figure \ref{fig:mst_rand_gen} shows qualitative results for the random generation of MNIST and SVHN samples.
\begin{figure*}
    \centering
    \begin{subfigure}[b]{0.3\textwidth}
        \centering
        \includegraphics[width=1.0\textwidth]{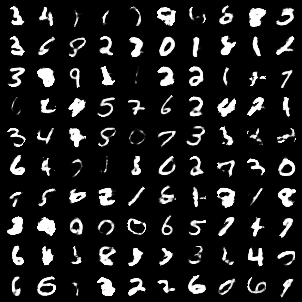}
        \caption{MVAE: MNIST}
        \label{fig:mnistsvhntext_mvae_rand_gen_mnist}
    \end{subfigure}
    \begin{subfigure}[b]{0.3\textwidth}
        \centering
        \includegraphics[width=1.0\textwidth]{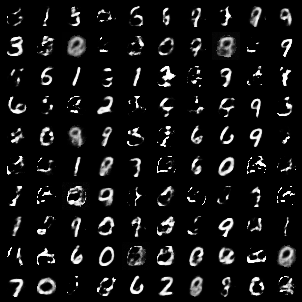}
        \caption{MMVAE: MNIST}
        \label{fig:mnistsvhntext_mmvae_rand_gen_mnist}
    \end{subfigure}
    \begin{subfigure}[b]{0.3\textwidth}
        \centering
        \includegraphics[width=1.0\textwidth]{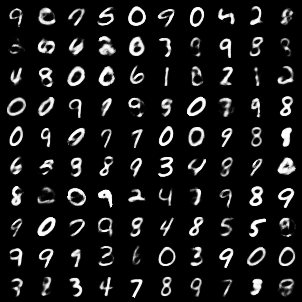}
        \caption{mmJSD: MNIST}
        \label{fig:mnistsvhntext_mmjsd_rand_gen_mnist}
    \end{subfigure}
    \vfill
    \begin{subfigure}[b]{0.3\textwidth}
        \centering
        \includegraphics[width=1.0\textwidth]{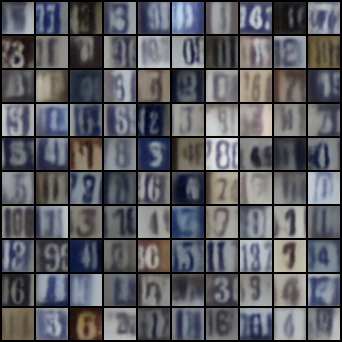}
        \caption{MVAE: SVHN}
        \label{fig:mnistsvhntext_mvae_rand_gen_svhn}
    \end{subfigure}
    \begin{subfigure}[b]{0.3\textwidth}
        \centering
        \includegraphics[width=1.0\textwidth]{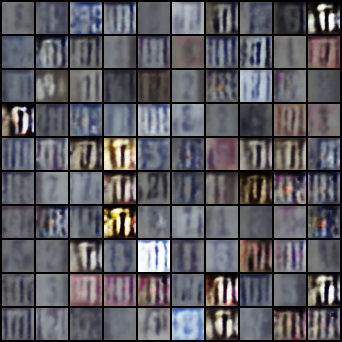}
        \caption{MMVAE: SVHN}
        \label{fig:mnistsvhntext_mmvae_rand_gen_svhn}
    \end{subfigure}
    \begin{subfigure}[b]{0.3\textwidth}
        \centering
        \includegraphics[width=1.0\textwidth]{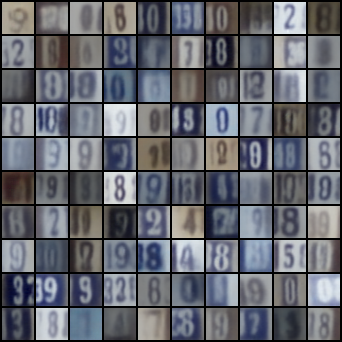}
        \caption{mmJSD: SVHN}
        \label{fig:mnistsvhntext_mmjsd_rand_gen_svhn}
    \end{subfigure}
    \centering
    \caption{Qualitative results for random generation.}
    \label{fig:mst_rand_gen}
\end{figure*}

\subsubsection{Comparison to Shi et al.}
\label{subsubsec:app_comparison_shi}
The results reported in \citet{shi2019variational}'s paper with the MMVAE model rely heavily on importance sampling (IS) (as can be seen by comparing to the numbers of a model without IS reported in their appendix). The IS-based objective \cite{Burda2015} is a different objective and difficult to compare to models without an IS-based objective. Hence, to have a fair comparison between all models we compared all models without IS-based objective in the main paper. The focus of the paper was on the different joint posterior approximation functions and the corresponding ELBO which should reflect the problems of a multimodal model.

For completeness we compare the proposed model to the IS-based MMVAE model here in the appendix.
Table \ref{tab:mst_runtimes} shows the training times for the different models. Although the MMVAE (I=30) only needs 30 training epochs for convergence, these 30 epochs take approximately 3 times as long as for the other models without importance sampling. (I=30) names the model with 30 importance samples. What is also adding up to the training time for the MMVAE (I=30) model is the $M^2$ paths through the decoder.
The MMVAE model and mmJSD need approximately the same time until training is finished. MVAE takes longer as the training objective is a combination of ELBOs instead of a single objective.

\begin{table}
\caption{Comparison of training times on the MNIST-SVHN-Text dataset. (I=30) names the model with 30 importance samples.}
\label{tab:mst_runtimes}
\vskip 0.15in
\begin{center}
\begin{small}
\begin{sc}
\begin{tabular}{lcc}
    \toprule
    Model & \#epochs & runtime \\
    \midrule
    MVAE & 50 & 3h 01min \\
    MMVAE & 50 & 2h 01min \\
    MMVAE (I=30) & 30 & 15h 15min \\
    mmJSD & 50 & 2h 16min \\
    MVAE (MS) & 100 & 6h 15min\\
    MMVAE (MS) & 100 & 4h 10min \\
    mmJSD (MS) & 100 & 4h 36min \\
    \bottomrule
\end{tabular}
\end{sc}
\end{small}
\end{center}
\vskip -0.1in
\end{table}

\begin{table}
\caption{Classification accuracy of the learned latent representations using a linear classifier. We evaluate all subsets of modalities for which we use the following abbreviations: M: MNIST; S: SVHN; T: Text; M,S: MNIST and SVHN; M,T: MNIST and Text; S,T: SVHN and Text; Joint: all modalities. (MS) names the models with modality-specific latent subspaces. (I=30) names the model with 30 importance samples.}
\label{tab:mst_representation_is}
\vskip 0.15in
\begin{center}
\begin{small}
\begin{sc}
\begin{tabular}{lccccccc}
    \toprule
    Model &  M & S & T & M,S & M,T & S,T & Joint \\
    \midrule
    MMVAE & 0.96 & 0.81 & \textbf{0.99} & 0.89 & 0.97 & 0.90 & 0.93 \\
    MMVAE (I=30) & 0.92 & 0.67 & \textbf{0.99} & 0.80 & 0.96 & 0.83 & 0.86 \\
    \textbf{mmJSD} & 0.97 & 0.82 & \textbf{0.99} & 0.93 & \textbf{0.99} & 0.92 & 0.98 \\
    \midrule
    MMVAE (MS) & 0.96 & 0.81 & \textbf{0.99} & 0.89 & 0.98 & 0.91 & 0.92 \\
    \textbf{mmJSD (MS)} & \textbf{0.98} & \textbf{0.85} & \textbf{0.99} & \textbf{0.94} & 0.98 & \textbf{0.94} & \textbf{0.99} \\
    \bottomrule
\end{tabular}
\end{sc}
\end{small}
\end{center}
\vskip -0.1in
\end{table}

\begin{table}
\caption{Classification accuracy of generated samples on MNIST-SVHN-Text. In case of conditional generation, the letter above the horizontal line indicates the modality which is generated based on the different sets of modalities below the horizontal line. (I=30) names the model with 30 importance samples.}
\label{tab:mst_generation_coherence_is}
\vskip 0.15in
\begin{center}
\begin{small}
\begin{sc}
\begin{tabular}{lcccccccccc}
\toprule
&  & \multicolumn{3}{c}{M} & \multicolumn{3}{c}{S} & \multicolumn{3}{c}{T} \\
\cmidrule(l){3-5} \cmidrule(l){6-8} \cmidrule(l){9-11}
Model & Random & S & T & S,T & M & T & M,T & M & S & M,S \\
\midrule
MMVAE (I=30) & 0.60 & 0.71 & \textbf{0.99} & 0.85 & 0.76 & 0.68 & 0.72 & 0.95 & 0.73 & 0.84 \\
MMVAE & 0.54 & \textbf{0.82} & \textbf{0.99} & 0.91 & 0.32 & 0.30 & 0.31 & 0.96 & \textbf{0.83} & 0.90 \\
\textbf{mmJSD} & 0.60 & \textbf{0.82} & \textbf{0.99} & \textbf{0.95} & 0.37 & 0.36 & 0.48 & \textbf{0.97} & \textbf{0.83} & \textbf{0.92} \\
\midrule
MMVAE (MS) & \textbf{0.67} & 0.77 & 0.97 & 0.86 & 0.88 & \textbf{0.93} & 0.90 & 0.82 & 0.70 & 0.76 \\
\textbf{mmJSD (MS)} & 0.66 & 0.80 & 0.97 & 0.93 & \textbf{0.89} & \textbf{0.93} & \textbf{0.92} & 0.92 & 0.79 & 0.86 \\
\bottomrule
\end{tabular}
\end{sc}
\end{small}
\end{center}
\vskip -0.1in
\end{table}

\begin{table}[h]
\caption{Test set log-likelihood on MNIST-SVHN-Text. We report the log-likelihood of the joint generative model $p_\theta(\bm{X})$. (I=30) names the model with 30 importance samples.}
\label{tab:mst_likelihood_is}
\vskip 0.15in
\begin{center}
\begin{small}
\begin{sc}
\begin{tabular}{lc}
    \toprule
    Model & $\bm{X}$  \\
    \midrule
    MVAE & \textbf{-1864} \\
    MMVAE (I=30) & -1891 \\
    MMVAE & -1916 \\
    \textbf{mmJSD} & -1961 \\
    \midrule
    MVAE (MS) & -1870 \\
    MMVAE (MS) & -1893 \\
    \textbf{mmJSD (MS)} & -1900 \\
    \bottomrule
\end{tabular}
\end{sc}
\end{small}
\end{center}
\vskip -0.1in
\end{table}

Tables \ref{tab:mst_representation_is}, \ref{tab:mst_generation_coherence_is} and \ref{tab:mst_likelihood_is} show that the models without any importance samples achieve state-of-the-art performance compared to the MMVAE model using importance samples. Using modality-specific subspaces seems to have a similar effect towards test set log-likelihood performance as using importance samples with a much lower impact on computational efficiency as it can be seen in the comparison of training times in Table \ref{tab:mst_runtimes}.

\subsubsection{Modality-Specific Subspaces}
\label{subsec:mst_ms}
The introduction of modality-specific subspaces introduces an additional degree of freedom. In Figure \ref{fig:mst_abl_ms_dim}, we show a comparison of different modality-specific subspace sizes. The size is the same for all modalities. Also, the total number of latent dimensions is constant, i.e. the number of dimensions in the modality-specific subspaces is subtracted from the shared latent space. If we have modality-specific latent spaces of size 2, the shared latent space is of size 18. This allows to ensure that the capacity of latent spaces stays constant. Figure \ref{fig:mst_abl_ms_dim} shows that the introduction of modality-specific subspaces only has minor effect on the quality of learned representations, despite the lower number of dimensions in the shared space. Generation coherence suffers with increasing number of modality-specific dimensions, but the quality of samples improves. We guess that the coherence becomes lower due to information which is shared between modalities but encoded in modality-specific spaces. In future work, we are interested in finding better schemes to identify shared and modality-specific information.

\begin{figure*}
    \centering
    \begin{subfigure}[b]{0.45\textwidth}
        \centering
        \includegraphics[width=1.0\textwidth]{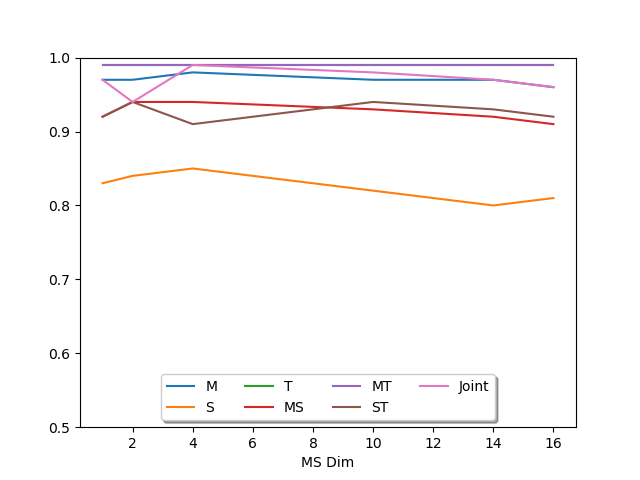}
        \caption{Latent Representation Classification}
        \label{fig:mst_abl_ms_dim_rep}
    \end{subfigure}
    \begin{subfigure}[b]{0.45\textwidth}
        \centering
        \includegraphics[width=1.0\textwidth]{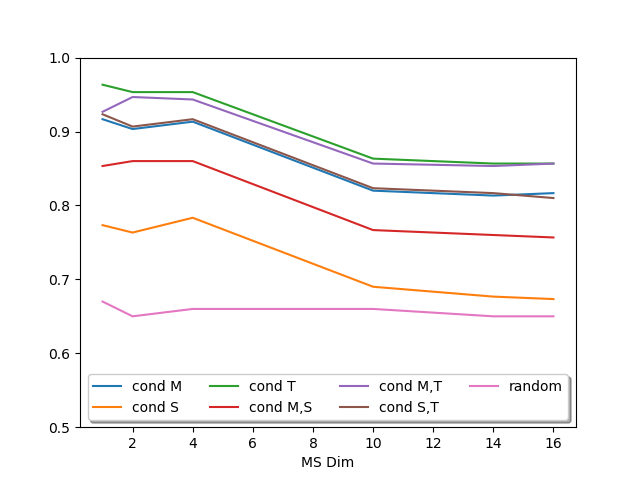}
        \caption{Generation Coherence}
        \label{fig:mst_abl_ms_dim_coherence}
    \end{subfigure}
    \begin{subfigure}[b]{0.45\textwidth}
        \centering
        \includegraphics[width=1.0\textwidth]{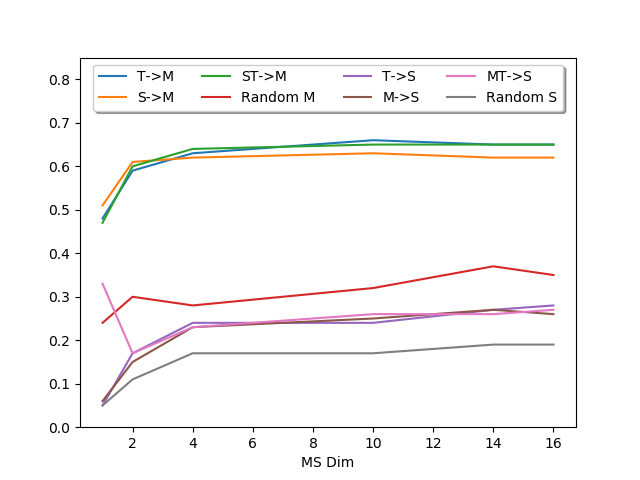}  
        \caption{Quality of Samples}
        \label{fig:mst_abl_ms_dim_prd}
    \end{subfigure}
    \centering
    \caption{Comparison of different modality-specific latent space sizes for the proposed mmJSD objective.}
    \label{fig:mst_abl_ms_dim}
\end{figure*}

\subsubsection{Weight of predefined distribution in JS-divergence}
\label{subsec:mst_weight_dp}
We empirically analyzed the influence of different weights of the pre-defined distribution $p_\theta(\bm{z})$ in the JS-divergence. Figure \ref{fig:mst_abl_dp_weight} shows the results. We see the constant performance regarding the latent representations and the quality of samples. In future work we would like to study the drop in performance regarding the coherence of samples if the weight of the pre-defined distribution $p_\theta(\bm{z})$ is around 0.4.

\begin{figure*}
    \centering
    \begin{subfigure}[b]{0.45\textwidth}
        \centering
        \includegraphics[width=1.0\textwidth]{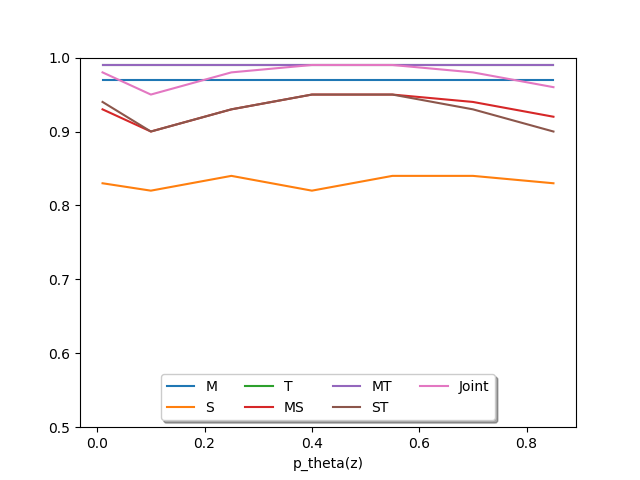}
        \caption{Latent Representation Classification}
        \label{fig:mst_abl_dp_weight_rep}
    \end{subfigure}
    \begin{subfigure}[b]{0.45\textwidth}
        \centering
        \includegraphics[width=1.0\textwidth]{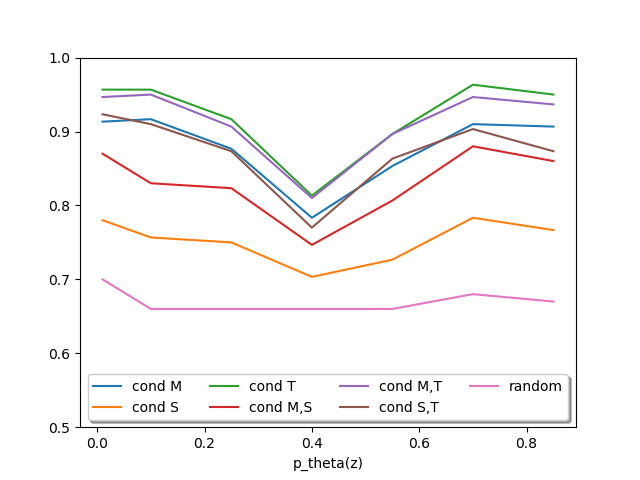}
        \caption{Generation Coherence}
        \label{fig:mst_abl_dp_weight_coherence}
    \end{subfigure}
    \begin{subfigure}[b]{0.45\textwidth}
        \centering
        \includegraphics[width=1.0\textwidth]{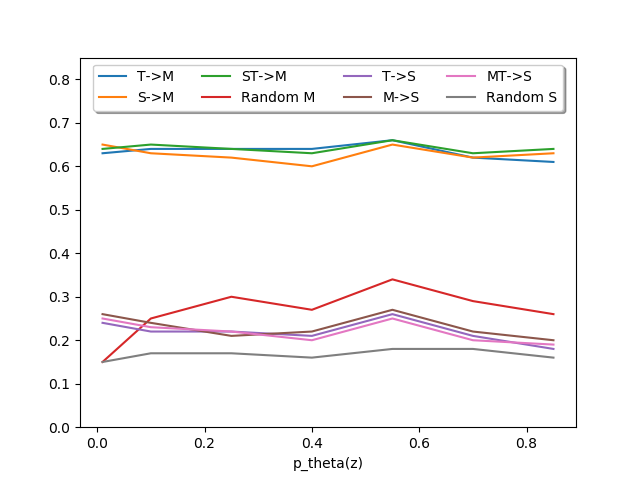}  
        \caption{Quality of Samples}
        \label{fig:mst_abl_dp_weight_prd}
    \end{subfigure}
    \centering
    \caption{Comparison of different weights for the pre-defined distribution $p_\theta(\bm{z})$ in the JS-divergence.}
    \label{fig:mst_abl_dp_weight}
\end{figure*}

\subsection{CelebA}
\label{subsec:app_celeba}
\subsubsection{Bimodal Dataset}
\label{sec:app_celeba_dataset}
Every face in the dataset is labelled with 40 attributes. For the text modality, we create text strings from these attributes. The text modality is a concatenation of available attributes into a comma-separated list. Underline characters are replaced by a blank space. We create strings of length 256 (which is the maximum string length possible following described rules). If a given face has only a small number of attributes which would result in a short string, we fill the remaining space with the asterix character $\ast$. Table \ref{tab:app_celeba_string_examples} shows examples of strings.
\begin{table*}
\caption{Examples of strings we created to have a bimodal version of CelebA which results in pairs of images and texts. For illustrative reasons we dropped the asterix characters.}\smallskip
\centering
\resizebox{.95\textwidth}{!}{
\smallskip
    \begin{tabular}{l}
        bags under eyes, chubby, eyeglasses, gray hair, male, mouth slightly open, oval face, sideburns, smiling, straight hair \\
        big nose, male, no beard, young \\
        attractive, big nose, black hair, bushy eyebrows, high cheekbones, male, mouth slightly open, no beard, oval face, smiling, young\\
        5 o clock shadow, bags under eyes, big nose, bushy eyebrows, chubby, double chin, gray hair, high cheekbones, male, mouth slightly open, no beard, smiling, straight hair, wearing necktie \\
        arched eyebrows, attractive, bangs, black hair, heavy makeup, high cheekbones, mouth slightly open, no beard, pale skin, smiling, straight hair, wearing lipstick, young \\
        attractive, brown hair, bushy eyebrows, high cheekbones, male, no beard, oval face, smiling, young\\
        attractive, high cheekbones, no beard, oval face, smiling, wearing lipstick, young \\
        attractive, blond hair, heavy makeup, high cheekbones, mouth slightly open, no beard, oval face, smiling, wearing lipstick, young \\
        attractive, brown hair, heavy makeup, no beard, oval face, pointy nose, straight hair, wearing lipstick, young\\
        5 o clock shadow, bags under eyes, big nose, brown hair, male, mouth slightly open, smiling, young \\
        attractive, brown hair, heavy makeup, high cheekbones, mouth slightly open, no beard, oval face, pointy nose, smiling, wavy hair, wearing earrings, wearing lipstick, young \\
        attractive, bangs, blond hair, heavy makeup, high cheekbones, mouth slightly open, no beard, oval face, smiling, wavy hair, wearing earrings, wearing lipstick, young \\
    \end{tabular}
}
\label{tab:app_celeba_string_examples}
\end{table*}
\subsubsection{Implementation Details}
For the CelebA experiments, we switched to a ResNet architecture \cite{he2016deep} for encoders and decoders of image and text modality due to the difficulty of the dataset. The specifications of the individual layers for the image and text networks can be found in Tables \ref{tab:celeba_image_layers} and \ref{tab:celeba_text_layers}.
The image modality is modelled with a Laplace likelihood and a Gaussian distributed posterior approximation. The text modality is modelled with a categorical likelihood and a Gaussian distributed posterior approximation. Their likelihoods are weighted according to the data size with the image likelihood being set to 1.0. The text likelihood is scaled according to $size(Img)/size(Text)$. The global $\beta$ is set to 2.5 and the $\beta_S$ of the modality-specific subspaces again to the number of modalities, i.e. 2.
The shared as well as the modality-specific latent spaces consist all of 32 dimensions.
For training, we used a batch size of $256$. We use ADAM as optimizer \cite{kingma2014adam} with a starting learning rate of 0.001. We trained our model for 100 epochs.
\begin{table*}
    \caption{CelebA Image: Encoder and Decoder Layers. The specifications name kernel size, stride, padding and dilation. res names a residual block.}
    \label{tab:celeba_image_layers}
    \vskip 0.15in
    \begin{center}
    \begin{small}
    \begin{tabular}{llccc|llccc}
        \toprule
        \multicolumn{5}{c}{Encoder} & \multicolumn{5}{c}{Decoder} \\
        Layer & Type & \#F. In & \#F. Out & Spec. & Layer & Type & \#F. In & \#F. Out & Spec.\\
        \midrule
        1 & conv$_{2d}$ & 3 & 128 & (3, 2, 1, 1) & 1 & linear & 64 & 640 & \\
        2 & res$_{2d}$ & 128 & 256 & (4, 2, 1, 1) & 2 & res$_{2d}^{T}$ & 640 & 512 & (4, 1, 0, 1) \\
        3 & res$_{2d}$ & 256 & 384 & (4, 2, 1, 1) & 3 & res$_{2d}^{T}$ & 512 & 384 & (4, 1, 1, 1) \\
        4 & res$_{2d}$ & 384 & 512 & (4, 2, 1, 1) & 4 & res$_{2d}^{T}$ & 384 & 256 & (4, 1, 1, 1) \\
        5 & res$_{2d}$ & 512 & 640 & (4, 2, 1, 1) & 5 & res$_{2d}^{T}$ & 256 & 128 & (4, 1, 1, 1) \\
        6a & linear & 640 & 32 & & 6 & conv$_{2d}^{T}$ & 128 & 3 & (3, 2, 1, 1) \\
        6b & linear & 640 & 32 &  &  &  & \\
        \bottomrule
    \end{tabular}
    \end{small}
    \end{center}
    \vskip -0.1in
\end{table*}
\begin{table*}
    \caption{CelebA Text: Encoder and Decoder Layers. The specifications name kernel size, stride, padding and dilation. res names a residual block.}
    \label{tab:celeba_text_layers}
    \vskip 0.15in
    \begin{center}
    \begin{small}
    \begin{tabular}{llccc|llccc}
        \toprule
        \multicolumn{5}{c}{Encoder} & \multicolumn{5}{c}{Decoder} \\
        Layer & Type & \#F. In & \#F. Out & Spec. & Layer & Type & \#F. In & \#F. Out & Spec.\\
        \midrule
        1 & conv$_{1d}$ & 71 & 128 & (3, 2, 1, 1) & 1 & linear & 64 & 896 &  \\
        2 & res$_{1d}$ & 128 & 256 & (4, 2, 1, 1) & 2 & res$_{1d}^{T}$ & 640 & 640 & (4, 2, 0, 1) \\
        3 & res$_{1d}$ & 256 & 384 & (4, 2, 1, 1) & 3 & res$_{1d}^{T}$ & 640 & 640 & (4, 2, 1, 1) \\
        4 & res$_{1d}$ & 384 & 512 & (4, 2, 1, 1) & 4 & res$_{1d}^{T}$ & 640 & 512 & (4, 2, 1, 1) \\
        5 & res$_{1d}$ & 512 & 640 & (4, 2, 1, 1) & 5 & res$_{1d}^{T}$ & 512 & 384 & (4, 2, 1, 1) \\
        6 & res$_{1d}$ & 640 & 640 & (4, 2, 1, 1) & 6 & res$_{1d}^{T}$ & 384 & 256 & (4, 2, 1, 1) \\
        7 & res$_{1d}$ & 640 & 640 & (4, 2, 0, 1) & 7 & res$_{1d}^{T}$ & 256 & 128 & (4, 2, 1, 1) \\
        8a & linear & 640 & 32 &  & 8 & conv$_{1d}^{T}$ & 128 & 71 & (3, 2, 1, 1)\\
        8b & linear & 640 & 32 &  &  &  & \\
        \bottomrule
    \end{tabular}
    \end{small}
    \end{center}
    \vskip -0.1in
\end{table*}
\subsubsection{Results}
In Figure \ref{fig:app_celeba_random_image_samples}, we show randomly generated images sampled from the joint latent distribution. Table \ref{tab:app_celeba_random_text_samples} shows the corresponding text samples of the first row in Figure \ref{fig:app_celeba_random_image_samples}.
Figures \ref{fig:celeba_rep} and \ref{fig:celeba_gen} show quantitative results which demonstrate the difficulty of this dataset. Figure \ref{fig:celeba_rep} show classification accuracies of the latent representation for the different attributes. Because of the imbalanced nature of some attributes, we report the average precision. This figure demonstrates the difficulty to learn a good latent representation for all attributes. A similar pattern can be seen in Figure \ref{fig:celeba_gen} which shows the classification performance of generated samples according to the different attributes. The distribution over classification performances of latent representations and conditionally generated samples is similar. This pattern gives further evidence on the importance of a good latent representation for coherent generation in case of missing data. Additionally, Figure \ref{fig:celeba_rep} and \ref{fig:celeba_gen} show the superior performance of the proposed mmJSD objective with respect to almost all attributes.

\begin{figure*}
    \centering
    \includegraphics[width=0.95\textwidth]{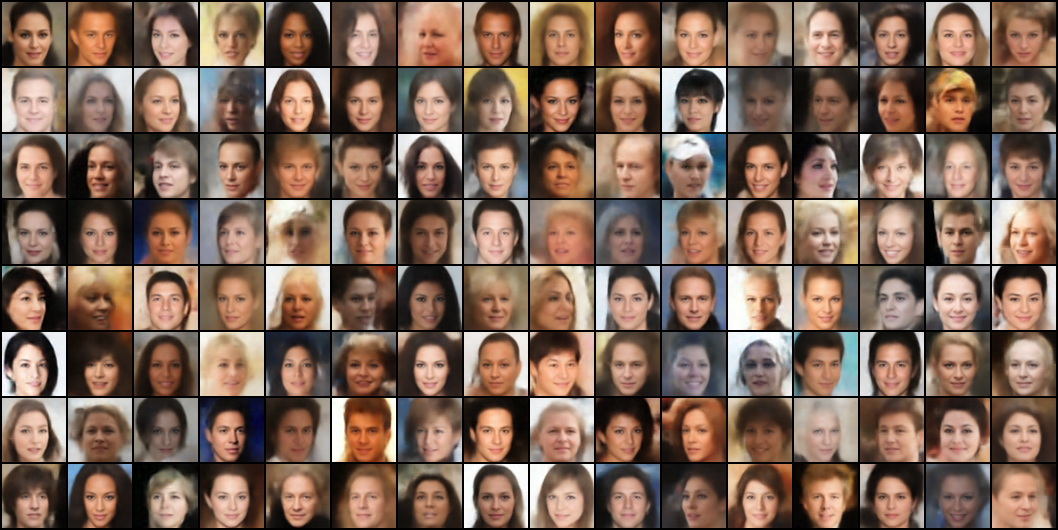}
    \caption{Randomly generated CelebA images sampled from the joint latent space of the proposed model.}
    \label{fig:app_celeba_random_image_samples}
\end{figure*}
\begin{table*}
\caption{Randomly generated CelebA strings sampled from the joint latent space of the proposed model. The strings correspond to the first row of images in Figure \ref{fig:app_celeba_random_image_samples}. We cut after the remaining asterix characters for illustrative reasons.}\smallskip
\centering
\resizebox{.95\textwidth}{!}{
\smallskip
    \begin{tabular}{l}
        5 o clock shadow, arched eyebrows, attig lips, blldn ha, big n ee, basd, your, ho beari, ntraight hair, wyanose, smiling, wearins eactt* *** blg******ing**** \\
        bangs, big lips, brown hair, gray hair, male, no beard, woung**********************************************************************************\\
        arched eyebrows, attractive, bengy  blbcows, heuvy eabones, mouth slig tlyr, narrow eyes, no beard, smiling, posniyhngiewavy hang, young*******\\
        bangs, big lips, black hair, high cheekbones, mouth slightly open, no beard, pale skin, wavy hang, young***************************************\\
        big lips, big nose, black hair, bushy ey, high cheekbones, narrow eyes, noface, pointy nose, smiling, wavy hair, young*************************\\
        bags under eyes, mouth slightly open, no beard, smiface, straight hairsmilirair,traight hair, young********************************************\\
        attractive, blond hair, brown hhigh chee aoses, mouth slightly open, no beard, oval facg, young************************************************\\
        arched eyebrows, bags under eyes, blackose, black h ir, ch ebes, narrow eyep, no  eard, wavy hair, wearing lipstick, young*********************\\
        big nose, blond eyebrows, no bmale, s, no beard, wavy hair, young******************************************************************************\\
        attractive, black hair, heavy makeup, high cheekbones, no beard, smiling, wearing lipstick, young**********************************************\\
        5 o clock shadow, bags under eyes, bald, mase, mou hegh arrow eyes, no beard, straight hair, wearing lipstick, young***************************\\
        black hair, blurry, brown hair,p,  o albeard, smiling******************************************************************************************\\
        attractive, black hair, brown hair, maatbe, mals, no beard, rosy ling, w smiling***************************************************************\\
        arched eyebrows, attractive, brown hair, bl ngwe, weari, youtg*********************************************************************************\\
        big lips, eyeglasses, high, no bea d, yeang, young*********************************************************************************************\\ 
        bangs, brown hair, byehlasses, ws,vmouth sl, no beard, oval facd, smiling, wearing lipstick, young*********************************************\\
    \end{tabular}
}
\label{tab:app_celeba_random_text_samples}
\end{table*}

\begin{figure*}[ht]
\vskip 0.2in
    \centering
    \begin{subfigure}[b]{0.95\textwidth}
        \centering
        \includegraphics[width=1.0\textwidth]{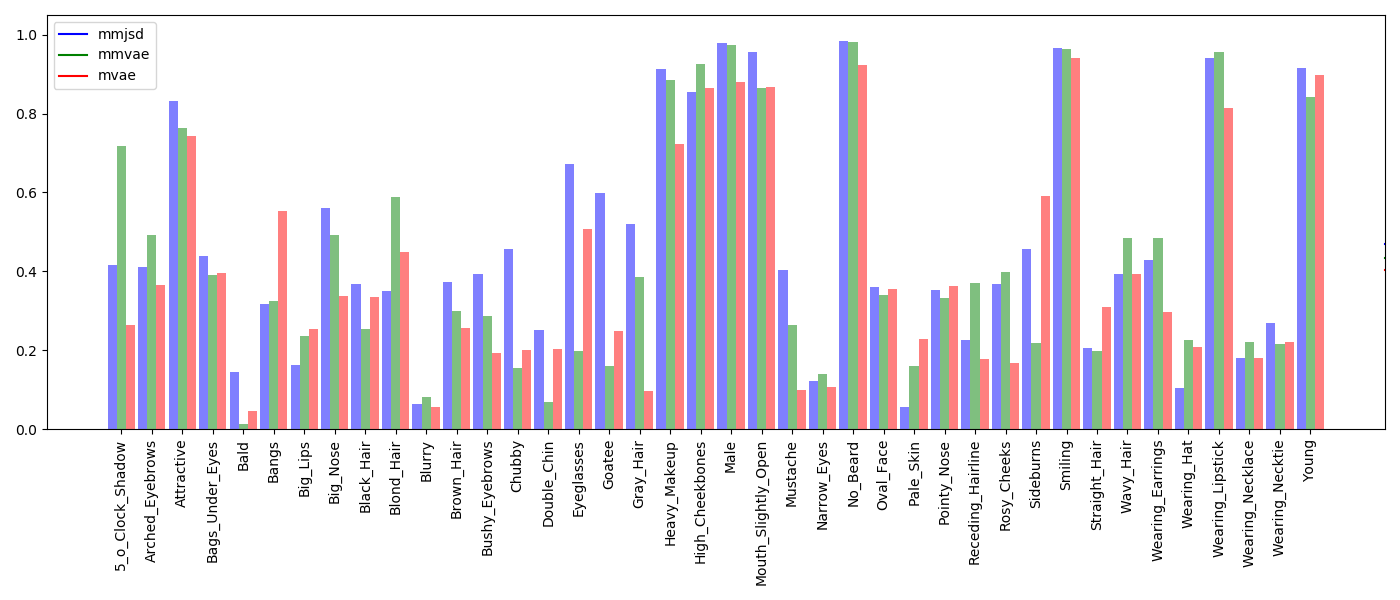}
        \caption{Img}
        \label{fig:celeba_rep_img}
    \end{subfigure}
    \hfill
    \begin{subfigure}[b]{0.95\textwidth}
        \centering
        \includegraphics[width=1.0\textwidth]{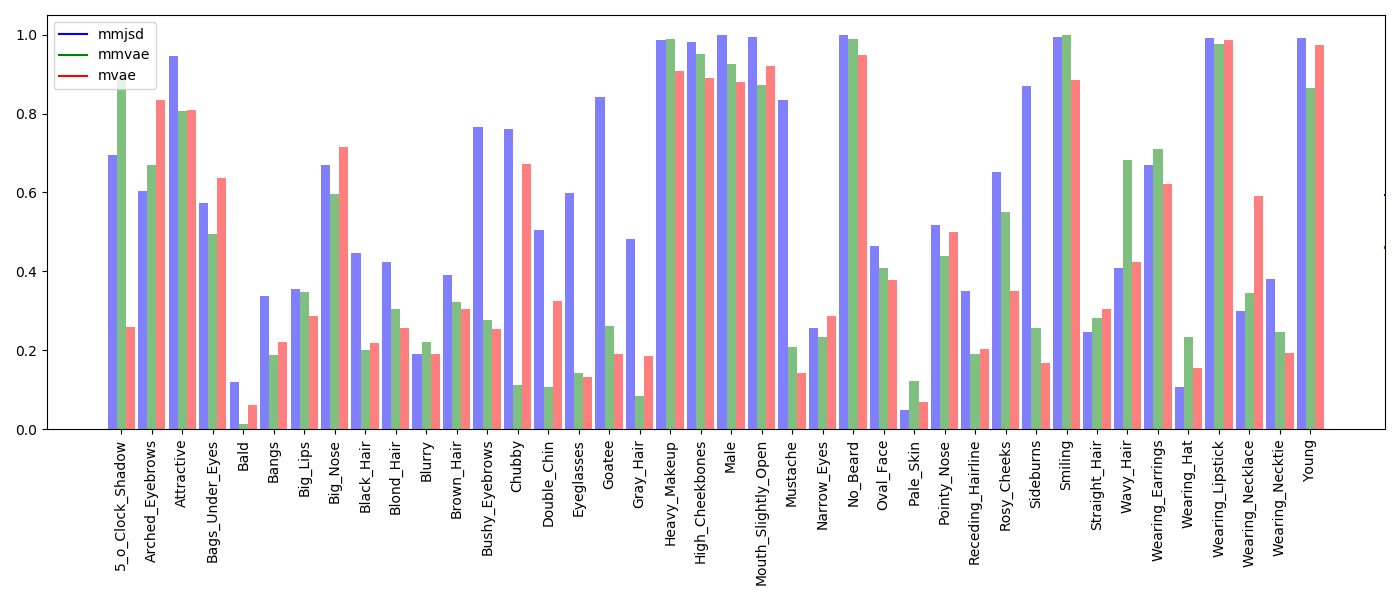}
        \caption{Text}
        \label{fig:celeba_rep_text}
    \end{subfigure}
    \hfill
    \centering
    \begin{subfigure}[b]{0.95\textwidth}
        \centering
        \includegraphics[width=1.0\textwidth]{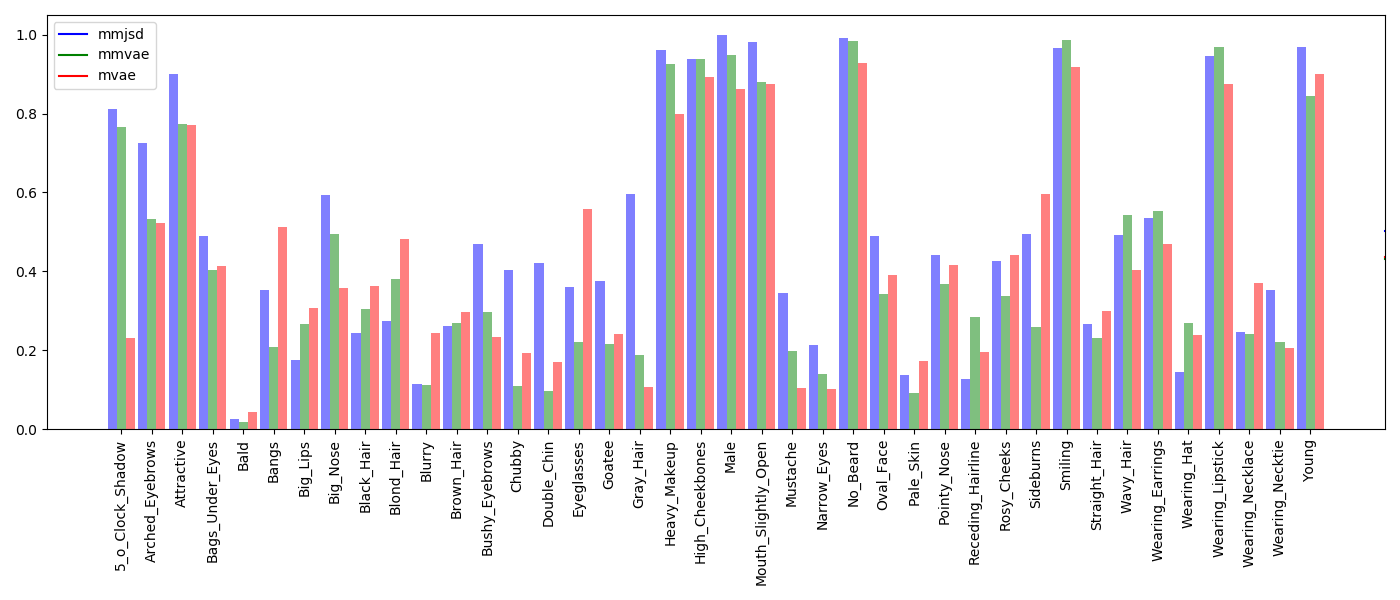}
        \caption{Joint}
        \label{fig:celeba_rep_joint}
    \end{subfigure}
    \caption{Classification of learned representations on CelebA. We report the average precision (higher is better). The difficulty of learning the individual attributes can be seen by the difference in classification performance across attributes. On the other hand, performance distribution over attributes is similar for both modalities. For all subsets of modalities, the proposed mmJSD objective outperforms previous work.}
    \label{fig:celeba_rep}
\vskip -0.2in
\end{figure*}

\begin{figure*}[ht]
\vskip 0.2in
    \centering
    \begin{subfigure}[b]{0.95\textwidth}
        \centering
        \includegraphics[width=1.0\textwidth]{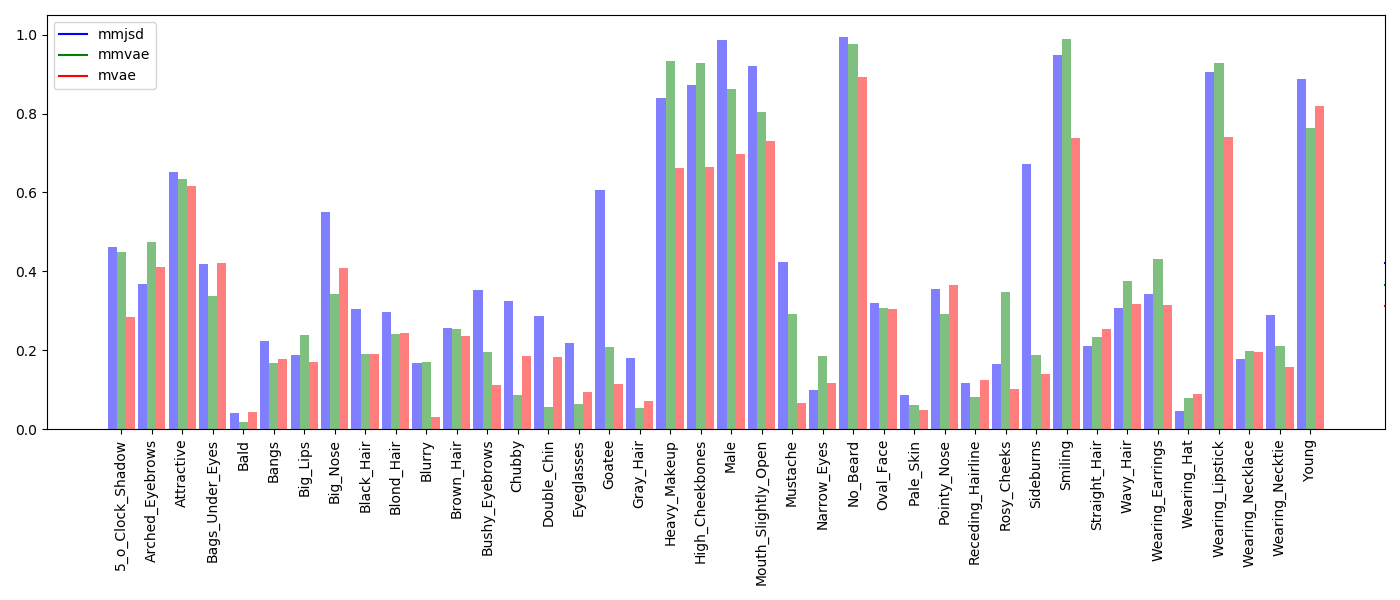}
        \caption{Img}
        \label{fig:celeba_gen_img}
    \end{subfigure}
    \hfill
    \begin{subfigure}[b]{0.95\textwidth}
        \centering
        \includegraphics[width=1.0\textwidth]{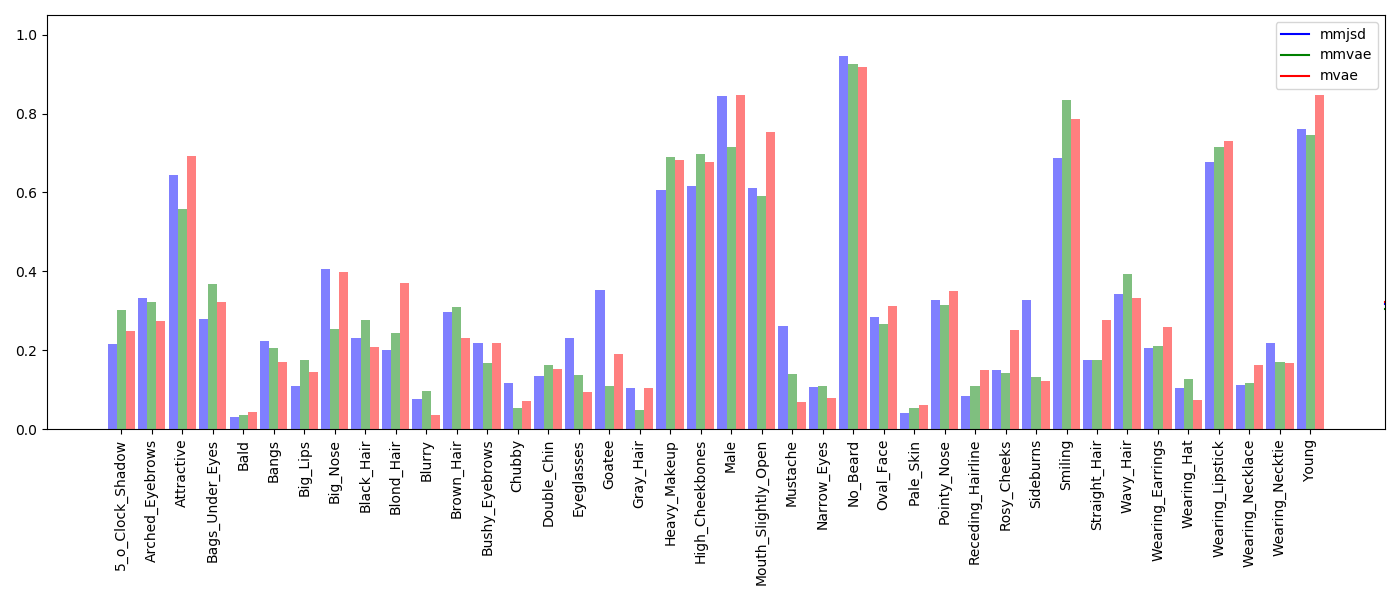}
        \caption{Text}
        \label{fig:celeba_gen_text}
    \end{subfigure}
    \caption{Classification accuracies of generated samples on CelebA. Coherent generation is mostly only possible if a linearly separable representation of an attribute is learned (see Figure~\ref{fig:celeba_rep}). The proposed mmJSD method achieves state-of-the-art or superior performance in the generation of both modalities. Img stands for images which are generated conditioned on text sample, Text for texts which are generated based on image samples.}
    \label{fig:celeba_gen}
\vskip -0.2in
\end{figure*}

\end{document}